\documentclass{article}

\usepackage{microtype}
\usepackage{graphicx}
\usepackage{wrapfig}
\usepackage{booktabs} 

\usepackage{hyperref}

\usepackage{enumitem}

\usepackage[linesnumbered, ruled, vlined]{algorithm2e}
\usepackage{algorithmic}

\usepackage{wasysym}
\newcommand{\param}{\theta} 
\newcommand{\fairnessloss}{L_\fairnessaggweight} 
\newcommand{\fairnessaggweight}{\beta} 

\newcommand{\loss}{L}
\newcommand{\reweight}{w}
\newcommand{\lsmooth}{C} 
\newcommand{\metalr}{\eta} 
\newcommand{\bdd}{\sigma} 

\newcommand{\alternative}{A} 
\newcommand{\dispoint}{D} 
\newcommand{\feasresult}{S} 
\newcommand{\disresult}{d} 

\newcommand{\vallosstmp}{L^{(val)}}
\newcommand{\trainlosstmp}{L^{(train)}}

\newcommand{\traindata}{D^{(train)}}
\newcommand{\valdata}{D^{(val)}}
\newcommand{\intermweight}{{\Tilde{\reweight}}} 
\newcommand{\backwardad}{\text{\textbf{BackwardAD}}} 
\newcommand{\normalize}{\text{\textbf{Normalize}}}
\newcommand{\minibatch}{D} 
\newcommand{\lr}{\eta}
\newcommand{\groupvalloss}{L^{(val)}} 
\newcommand{\groupgrad}{g}
\newcommand{\utility}{u} 
\newcommand{\groupgradagg}{\nabla \fairnessloss}
\newcommand{\gradupd}{\nabla \loss_\gcoefunknown} 
\newcommand{\radius}{\epsilon} 
\newcommand{\ball}{B_{\radius}} 
\newcommand{\alignangle}{\delta} 
\newcommand{\allgrad}{G} 
\newcommand{\groupcnt}{K} 
\newcommand{\paramdim}{d} 
\newcommand{\unknown}{x} 

\newcommand{\const}{\gamma} 
\newcommand{\gcoefunknown}{\alpha}  

\newcommand{\funcf}{f}
\newcommand{\funcg}{g} 
\newcommand{\paramspace}{\Theta}

\usepackage{PRIMEarxiv}

\usepackage{amsmath}
\usepackage{subcaption}
\usepackage{amssymb}
\usepackage{mathtools}
\usepackage{amsthm}
\usepackage{multirow}
\usepackage[symbol]{footmisc}
\usepackage{pifont}
\usepackage{comment}

\usepackage{bbm}
\usepackage{authblk}

\usepackage{xcolor}
\usepackage{colortbl}
\usepackage{graphicx}
\usepackage{array}
\usepackage{adjustbox}
\definecolor{lightgray}{gray}{0.9}
\definecolor{myorange}{HTML}{ff9900}
\definecolor{myred}{HTML}{ff0000}
\definecolor{myblue}{HTML}{027BF2}
\definecolor{deepred}{rgb}{0.631,0.102,0.102}
\definecolor{mildyellow}{HTML}{FFF2CC}

\usepackage[capitalize,noabbrev]{cleveref}

\usepackage{wrapfig}

\theoremstyle{plain}
\newtheorem{theorem}{Theorem}[section]

\newtheorem{corollary}[theorem]{Corollary}
\theoremstyle{definition}

\theoremstyle{remark}

\usepackage[textsize=tiny]{todonotes}

\newcolumntype{C}[1]{>{\centering\arraybackslash}p{#1}}

\newcommand{\yi}[1]{\textbf{\textcolor{red}{[Yi: #1]}}}

\newcommand{\xuelin}[1]{\textbf{\textcolor{cyan}{[Xuelin: #1]}}}

\title{Fairness-Aware Meta-Learning via Nash Bargaining}

\author[1]{Yi Zeng\thanks{Xuelin and Yi contributed equally. Yi's work is partially done at Responsible AI, Meta AI. Corresponding: \href{mailto:lichen66@meta.com}{Li Chen} and \href{mailto:ruoxijia@vt.edu}{Ruoxi Jia}. Code:
\href{https://github.com/ruoxi-jia-group/Nash-Meta-Learning}{Github: Nash-Meta-Learning}.}
}
\author[2]{Xuelin Yang$^*$}
\author[3]{Li Chen}
\author[3]{Cristian Canton Ferrer}
\author[1]{Ming Jin}
\author[2]{\\Michael I. Jordan}
\author[1]{Ruoxi Jia}
\affil[1]{Virginia Tech, Blacksburg, VA 24061, USA}
\affil[2]{University of California, Berkeley, CA 94720, USA}
\affil[3]{Meta AI, Menlo Park, CA 94025, USA}

\hypersetup{
  colorlinks=true,
  linkcolor=blue,      
  citecolor=blue,     
  urlcolor=deepred         
}

\begin{document}

\maketitle

\begin{abstract}
To address issues of group-level fairness in machine learning, it is natural to adjust model parameters based on specific fairness objectives over a sensitive-attributed validation set. Such an adjustment procedure can be cast within a meta-learning framework. However, naive integration of fairness goals via meta-learning can cause hypergradient conflicts for subgroups, resulting in unstable convergence and compromising model performance and fairness. To navigate this issue, we frame the resolution of hypergradient conflicts as a multi-player cooperative bargaining game. We introduce a two-stage meta-learning framework in which the first stage involves the use of a \textit{Nash Bargaining Solution} (NBS) to resolve hypergradient conflicts and steer the model toward the Pareto front, and the second stage optimizes with respect to specific fairness goals.
Our method is supported by theoretical results, notably a proof of the NBS for gradient aggregation free from linear independence assumptions, a proof of Pareto improvement, and a proof of monotonic improvement in validation loss. We also show empirical effects across various fairness objectives in six key fairness datasets and two image classification tasks. 
\end{abstract}

\begin{figure}[tb]
    \centering
    \includegraphics[width=\linewidth]{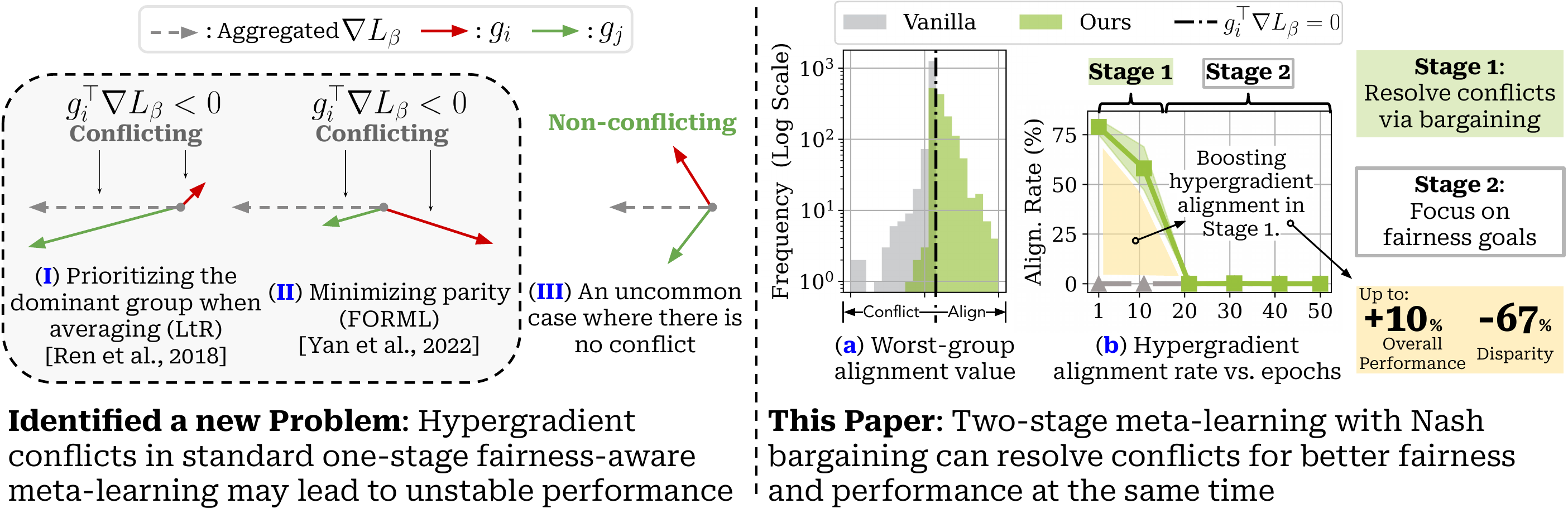}
    \caption{
    \textbf{Overview:} We illustrate the problem of \textit{hypergradient conflicts} in conventional one-stage fairness-aware meta-learning, which we find can lead to erratic performance and/or convergence at suboptimal, unfair local minima. 
    \textbf{Left:} Graphical depiction of group-wise hypergradient conflicts (\textcolor{blue}{\textbf{I}}, \textcolor{blue}{\textbf{II}}) showing different scenarios where conflicts arise in one-stage meta-learning, affecting performance stability; (\textcolor{blue}{\textbf{III}}) provides a depiction of the contrast case where the aggregated direction is not conflicting with any of the groups, which leads to a more stable, fair, and performant model.
    \textbf{Right:} (\textcolor{blue}{\textbf{a}}, \textcolor{blue}{\textbf{b}}) Comparison of traditional one-stage meta-learning (Vanilla, highlighted in \scalebox{0.8}{\colorbox[HTML]{e8e8e8}{gray}}) with our proposed two-stage meta-learning approach, which resolves inter-group hypergradient conflicts through a bargaining process (Ours, highlighted in \scalebox{0.8}{\colorbox[HTML]{dfecc3}{green}}). In our evaluation, we show the efficacy of our method in enhancing fairness-aware meta-learning, with improvements in performance by up to 10\% and fairness by up to 67\%, by initially focusing on conflict resolution in Stage 1 to steer the model towards the Pareto front followed by focusing on fairness goals in Stage 2.
    }
    \label{fig:teaser}
\end{figure}

\section{Introduction}

The traditional formulation of machine learning is in terms of a system that improves its predictive and decision-making performance by interacting with an environment.  Such a formulation is overly narrow in emerging applications---it lumps the social context of a learning system into the undifferentiated concept of an ``environment'' and provides no special consideration of the collective nature of learning.  Such social context includes notions of scarcity and conflict, as well as goals such as social norms and collaborative work that are best formulated at the level of social collectives.  The neglect of such considerations in traditional machine learning leads to undesirable outcomes in real-world deployments of machine learning systems, including outcomes that favor particular groups of people over others \cite{sagawa2019distributionally,caton2020fairness, mehrabi2021survey, chen2023algorithmic,pagano2023bias,wan2023processing}, the amplification of social biases and stereotypes \cite{leslie2020understanding,diaz2023connecting, ueda2024fairness}, and an ongoing lack of clarity regarding issues of communication, trust, and fairness.

Our focus is the current paper is fairness, and we take a perspective on fairness that blends learning methodology with economic mechanisms.  The current favored methodology for addressing fairness recognizes that it is not a one-size-fits-all concept---different fairness notions are appropriate for different social settings \cite{verma2018fairness, mitchell2021algorithmic,wachter2021fairness}---and treats fairness via meta-learning ideas. Meta-learning is implemented algorithmically with the tools of bi-level optimization. Specifically, fairness-aware meta-learning employs outer optimization to align with a specific fairness goal over a small, demographically balanced validation set 
to adjust a set of hyperparameters, while the inner optimization minimizes the hyperparameter-adjusted training loss \cite{ren2018learning,wei2022comprehensive,yan2022forml}.
This approach addresses two central challenges in group-level algorithmic fairness. First, it can integrate distinct fairness goals into the outer optimization. This flexibility allows customization of the focus of objectives, including enhancing the averaged loss across demographic- and label-balanced groups \cite{ren2018learning} and minimizing disparities \cite{yan2022forml}. This is a conceptual improvement over methods confined to a single fairness objective \cite{hashimoto2018fairness,lahoti2020fairness}.
Additionally, meta-learning reduces the reliance on sensitive attribute labels in the training data. This circumvents the label dependence in conventional methods \cite{luong2011k,kamiran2012data, iosifidis2019fae} and addresses ethical concerns over the acquisition of sensitive attributes \cite{ashurst2023fairness}. 

Although these arguments suggest that a meta-learning approach is promising for group-level fairness, it stops short of providing an economic mechanism which embodies fairness in terms of allocations and the management of conflict.  Our work aims to bridge this gap by bringing a concept from economic mechanism design---that of \emph{Nash bargaining}---into contact with meta-learning.  
Specifically, in our initial empirical explorations of meta-learning algorithms, we found that performance and fairness can vary substantially according to the choice of fairness metric across different datasets, suggesting a form of conflict that is not being resolved effectively via basic meta-learning procedures.  Investigating further,
we identified a phenomenon that we term \emph{hypergradient conflict} which we believe is a pivotal factor in driving the contrast in effectiveness among different fairness goals when integrated with meta-learning.  Briefly, the aggregated gradient of the outer optimization objective (the \emph{hypergradient})
conflicts with the desired update associated with particular groups (Figure \ref{fig:teaser}). 
To address this, we propose a novel framework that resolves hypergradient conflicts as a cooperative bargaining game. Specifically, we present a two-stage meta-learning framework for fairness: first incorporate the \textit{Nash Bargaining Solution} (NBS) at an early training stage to mitigate conflicts and steer the model toward the Pareto front, and then engage in the pursuit of specified fairness goals. 

Our work also introduces a new derivation of the NBS for gradient aggregation, one that dispenses with gradient independence assumptions made in the past work so that it is applicable to broader contexts.  This derivation may be of independent interest, and accordingly we present material on game-theoretic justification, Pareto improvement, and monotonic improvement in validation loss that help to relate the NBS to our gradient-based learning setting.  Our analysis sheds light on the convergence exhibited in empirical studies and provides an understanding of how our method improves meta-learning for fairness.

Finally, we present a thoroughgoing set of empirical studies that evaluate our method using synthetic and real-world datasets, encompassing six key fairness datasets and two image classification tasks. As we'll show, our framework uniformly enhances the performance of one-stage meta-learning methods, yielding up to 10\% overall performance improvement and a 67\% drop in disparity (Figure \ref{fig:teaser}).
The remainder of this paper is structured as follows: Section \ref{sec:problem} describes the problem. Section \ref{sec:method} details our method, including the bargaining game formulation, solution, and theoretical analyses ($\S$\ref{subsec:nash_setup}-\ref{subsec:nash_sol_theory}), and our two-stage meta-learning framework, its theoretical foundations, and dynamical issues ($\S$\ref{subsec:nash_meta}-\ref{subsec:nash_meta_dynamics}). Section \ref{sec:evaluation} presents empirical results and analysis. Related work is deferred to Appendix \ref{app:related_work}.

\section{Problem Statement}
\label{sec:problem}

Let $\theta$ denote a vector of model parameters, let $\traindata$ denote the training data and let $\valdata$ a validation set, with $\groupcnt$ sensitive groups $\valdata_i$, for $i \in [\groupcnt]$. Let $\loss$ be a vector-valued function where each entry corresponds to the per-example conventional loss (e.g., cross-entropy). Define \emph{group-wise validation loss} $\groupvalloss$ as a vector of size $\groupcnt$ where $\groupvalloss_i = \frac{1}{|\valdata_i|}\loss(\valdata_i|\param^*(\reweight))^\top\mathbbm{1}$ (the averaged loss over samples in group $i$). 
Define $\fairnessloss = \fairnessaggweight^\top \groupvalloss$ as the fairness loss with vector $\fairnessaggweight$ of size $\groupcnt$ encoding the fairness objectives under consideration. Following established work \cite{ren2018learning,yan2022forml}, we target group-level fairness objectives via meta-learning as follows:
\begin{align}
\reweight^* &= \arg\min_{\reweight \geq 0} \fairnessloss(\valdata|\param^*(\reweight)), \label{eq:opt_outer}\\
\param^* (\reweight) &= \arg\min_\param \reweight \cdot \loss(\traindata|\param). \label{eq:opt_inner}
\end{align}
%
The vector $\reweight$ is a set of hyperparameters that reweight each training example in the current minibatch when updating $\param$, optimized on $\valdata$ to improve group-level fairness.

Existing fairness-aware meta-learning work can be characterized as different protocols for $\fairnessaggweight$. \textbf{LtR} \cite{ren2018learning} computes the average of all group-losses with demographic and label balanced $\valdata$ (i.e., $\fairnessaggweight_{\text{LtR}}=\frac{1}{\groupcnt} \mathbbm{1}$). \textbf{FORML} \cite{yan2022forml} calculates the difference between the maximum and minimum group-loss (i.e., $\fairnessaggweight_{\text{FORML}}$ has 1 for max, -1 for min, and 0 otherwise), emphasizing parity. Meanwhile, group-level Max-Min fairness inspired from \cite{rawls2001justice} focuses solely on the maximum group-loss \cite{sagawa2019distributionally}, yielding a procedure referred to as \textbf{Meta-gDRO}, with $\fairnessaggweight_{\text{Meta-gDRO}}$ set equal to for the max and zero otherwise.
Throughout the training process, the hypergradient of these methods, $\nabla_{\reweight} \fairnessloss(\valdata|\param^*(\reweight))$,
is derived by applying the aggregation protocol $\fairnessaggweight$ to the \emph{group-wise hypergradients}, $\nabla_{\reweight} \groupvalloss_i(\valdata|\param^*(\reweight))$.
Their training process, wherein the aggregated hypergradient is iteratively utilized to update $w$, is referred to as a \emph{one-stage method}. They differ from our \emph{two-stage method}, which employs distinct hypergradient aggregation rules at two separate stages.

\begin{figure}[bt]
    \centering
        \begin{subfigure}[t]{0.49\linewidth}
        \vspace{-7.35em}
        \centering
        \resizebox{\linewidth}{!}{
        \begin{tabular}{@{}C{0.em}C{3em}C{6em}C{6em}C{6em}@{}}

        \toprule
        \multicolumn{1}{r|}{}        & \textbf{Baseline} &  \cellcolor[HTML]{ffffff}\textbf{LtR}         & \cellcolor[HTML]{ffffff}\textbf{FORML}         & \cellcolor[HTML]{ffffff}\textbf{Meta-gDRO} \\ 
        
        
        \hline \midrule
        \multicolumn{5}{l}{\multirow{1}{*}{\textbf{Adult Income}~\cite{misc_adult_2}, (sensitive attribute: \textbf{race)}}}\\[0pt] \midrule

        \multicolumn{1}{r|}{\textbf{Overall AUC (↑)}} & {0.668}                                                   & \cellcolor[HTML]{ffffff}0.803 (+20\%)  & 0.710 (+6.2\%) & 0.775 (+16\%)               \\

        \multicolumn{1}{r|}{\textbf{Max-gAUCD (↓)}}   &{0.225}                                               & 0.090 (-96\%)   & \cellcolor[HTML]{ffffff}0.290 (+29\%)  & 0.163 (-28\%)                 \\
        
        \multicolumn{1}{r|}{\textbf{Worst-gAUC (↑)}}  & {0.544}                                                    & 0.755 (+38\%)  & 0.540 (-0.7\%)  & \cellcolor[HTML]{ffffff}0.694  (+28\%)          \\

        


        
        
        
        \hline \midrule
        \multicolumn{5}{l}{\multirow{1}{*}{\textbf{Titanic Survival}~\cite{titanic}, (sensitive attribute: \textbf{sex)}}}\\ [0pt]\midrule

        \multicolumn{1}{r|}{\textbf{Overall AUC (↑)}} & {0.972}                                                   & \cellcolor[HTML]{ffffff}0.967 (-0.5\%)  & 0.950 (-2.3\%) & 0.961 (-1.1\%)           \\

        \multicolumn{1}{r|}{\textbf{Max-gAUCD (↓)}}   &{0.056}                                               & 0.044 (-21\%)  & \cellcolor[HTML]{ffffff}0.033 (-41\%)  & 0.033 (-41\%)                  \\
        
        \multicolumn{1}{r|}{\textbf{Worst-gAUC (↑)}}  & {0.944}                                                    & 0.944 (-) & 0.933 (-1.2\%)  & \cellcolor[HTML]{ffffff}0.944 (-)           \\

        \bottomrule
        \end{tabular}
        }
        \caption{
        Comparison of the model performance of three conventional one-stage meta-learning methods (LtR, FORML, and Meta-gDRO) across various fairness notions (results averaged from 5 runs, \% relative to the baseline only using the training set w/o meta-learning)
        }
        \label{table:fairness_problem}
        \end{subfigure}%
    \hfill
        \begin{subfigure}[t]{0.49\linewidth}
            \centering
            \includegraphics[width=\linewidth]{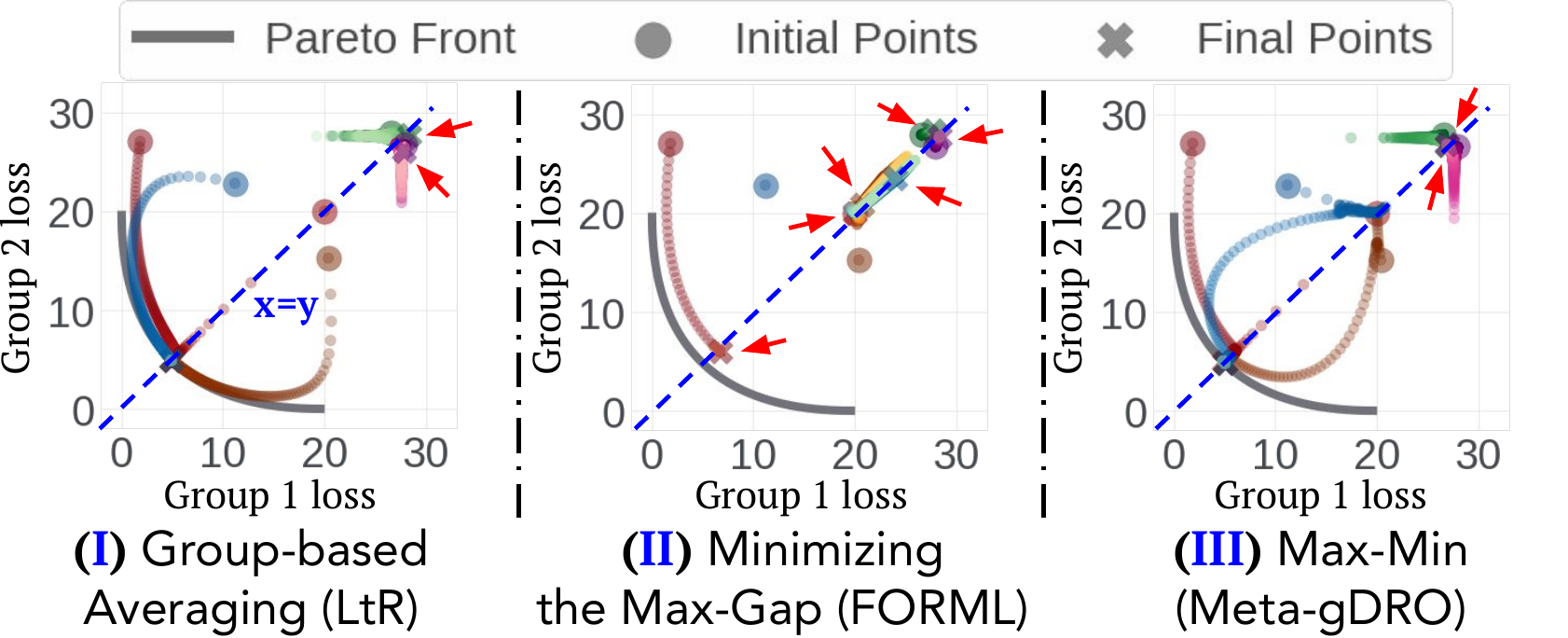}
            \caption{
            Trajectory of 1000-step optimizations. X and y-axis: validation losses for groups 1 and 2, rsp. Gray curves (lower left): the Pareto front. \textcolor{blue}{\textbf{x=y}}: fairness (equal validation loss). Ideal endpoint: their intersections. Undesirable endpoints: ``\textcolor{myred}{\ding{220}}''. 6 random initials. 
            }
            \label{fig:problem}
        \end{subfigure}
    \caption{The unreliable performance of conventional one-stage fairness-aware meta-learning.}
\end{figure}

\textbf{Unreliability of one-stage meta-learning for fairness.}
The traditional approach of plugging a fairness objective into $\fairnessloss$ seems natural. However, 
we find that the effect of this approach on performance and fairness can be unstable. We evaluate the above three one-stage fairness-aware methods based on targeted fairness metrics (detailed settings in $\S$\ref{subsec:fairness}): Overall Area Under the Curve (\textbf{Overall AUC}) for LtR which also measures prediction performance, maximum group AUC disparity (\textbf{Max-gAUCD}) for FORML, and worst group AUC (\textbf{Worst-gAUC}) for Meta-gDRO (Figure \ref{table:fairness_problem}).


\textbf{Delving into the cause.}
\label{subsec:motivation}
When inspecting the aggregated hypergradient with different one-stage meta-learning for fairness, we find that 
most epochs have \emph{less than 3\%} of the aggregated directions aligned to the optimization objectives of each subgroup. That is, we see \textit{hypergradient conflicts}:
\begin{align}
    \exists i, \text{ s.t. } \groupgrad_i^{\top} \groupgradagg < 0,
\end{align}
with the group-wise gradient $\groupgrad_i = \nabla_\reweight \groupvalloss_i(\valdata|\param(\reweight))$ and aggregated direction $\groupgradagg = \nabla_\reweight \fairnessloss(\valdata|\param(\reweight))$. 
The prevalence of intrinsic hypergradient conflict in one-stage methods is unsurprising, because their aggregation methods are unable to prevent overlooking or incorrectly de-prioritizing certain groups.
We study this phenomenon in synthetic settings to isolate structural issues from randomness in stochastic optimization (Figure \ref{fig:problem}, settings in Appendix \ref{app:exp_settings}): We define the performance goal as convergence to the Pareto front, while the fairness goal corresponds to the line $x=y$. As observed, alignment issues are prominent (Figure \ref{fig:problem}):
(\textcolor{blue}{\textbf{I}}) Averaging (LtR) may induce oscillatory dominance among groups.
(\textcolor{blue}{\textbf{II}}) Parity-focusing (FORML) leads to conflicting hypergradients due to one group's loss being subtracted, necessitating performance trade-offs for fairness.
(\textcolor{blue}{\textbf{III}}) Minimizing the worst-group loss (Meta-gDRO) often leads to toggling dominance, focusing only on the least-performing group's optimization direction, which can create conflicts and hinder progress toward the Pareto front.

\textbf{Hypergradient conflict resolution.}
%
Given the observation of hypergradient conflicts and convergence issues in one-stage meta-learning, we turn to cooperative bargaining and propose a two-stage method that seeks to attain more reliable improvements by resolving conflicts at the early stage of training.
Our methodology draws inspiration from the \textit{Nash Bargaining Solution} (NBS), a cornerstone of axiomatic bargaining in game theory, known for its general applicability and robustness \cite{nash1953two}. 
Nash Bargaining is chosen for its desirable axiomatic properties, which prohibit unconsented unilateral gains by Pareto Optimality, and its principled approach of effectively balancing interests, making it appealing for practical deployment. We provide additional discussion of the game-theoretic perspective and additional empirical comparisons in Appendix~\ref{app:discussion}.



While the NBS has been studied in multi-task learning \cite{navon2022multi}, it has yet to be explored in fairness-aware meta-learning. Unlike the settings in \cite{navon2022multi}, applying the NBS in our context challenges the assumption of linear independence among tasks, which is generally untenable for group-wise utility towards the same goal of performance gain (i.e., the settings of fairness). This drives our exploration into novel proofs and applications of the NBS in hypergradient aggregation, aiming to circumvent the need for linear independence and optimize shared outcomes through strategic negotiation and nested optimization.

\section{Methodology}
\label{sec:method}
\subsection{Nash Bargaining framework}
\label{subsec:nash_setup}
We start with some preliminaries. Consider $K$ players faced with a set $\alternative$ of alternatives. If all players reach a consensus on a specific alternative, denoted as $a$ in set $A$, then $a$ will be the designated outcome. In the event of a failure to agree on an outcome, a predetermined disagreement result, denoted as $d$, will be the final outcome. The individual utility payoff functions are denoted $\utility_i: \alternative \cup \{\dispoint\} \rightarrow \mathbb{R}$, which represent the players' preferences over $A$. Denote the set of feasible utility payoffs as $\feasresult = {(\utility_1(a), ..., \utility_K(a)) : a \in \alternative} \subset \mathbb{R}^K$ and the disagreement point as $\disresult = (\utility_1(\dispoint), ..., \utility_K(\dispoint))$. Nash proposed to study solutions to the bargaining problem through functions $f: (\feasresult, \disresult) \rightarrow \mathbb{R}$. The unique Nash bargaining solution (NBS), originally proposed for two players \cite{nash1953two} and latter extended to multiple players \cite{forgo1999introduction}, maximizes the function $f(\feasresult, \disresult)= \prod_i (x_i - \disresult_i)$, where $x_i$ is the bargained payoff and $\disresult_i$ is the disagreement payoff for player $i$. The NBS fulfills four axioms: Pareto Optimality, Symmetry, Independence of Irrelevant Alternatives, and Invariant to Affine Transformations. See Appendix \ref{app:prelim} for detailed definitions, and \ref{app:setup} for additional assumptions.

In our problem, denote that the hyperparameter of interest as $\intermweight$, an intermediate vector for optimal $\reweight^*$ in Algorithm \ref{alg:training} to reweight each training sample. Let $\groupvalloss_i$ be the validation loss and $\groupgrad_i = \nabla_\intermweight \groupvalloss_i$ be the hypergradient of group $i$. Let $G$ be the matrix with columns $\groupgrad_i$. We want to find the protocol $\gcoefunknown$ as the weights applied to individual group $i$'s loss. It associates an update step $\gradupd$ to $\intermweight$ that improves the aggregated validation loss among all groups.

We frame this problem as a cooperative bargaining game between the $\groupcnt$ groups. Define the utility function for group $i$ as
\begin{align}
  \utility_i(\gradupd) = \groupgrad_i^\top \gradupd . 
\end{align}
The intuition is that the utility tells us how much of proposed update is applied in the direction of hypergradient of group $i$ (as it can be written as $\lVert \groupgrad_i \rVert \lVert \gradupd \rVert \cos{\alignangle}$ with angle $\alignangle$ between $\groupgrad_i$ and $\gradupd$). This gives the projection of $\gradupd$ along $\groupgrad_i$, or the ``in effect'' update for group $i$. If $\groupgrad_i$ and $\gradupd$ aligns well, the utility of group $i$ is large (or \textit{vice versa}). 
Denote $\ball$ the ball of radius $\radius$ centered at 0. We are interested in the update $\gradupd$ in the agreement alternative set $\alternative = \{ \nabla L:  \nabla L \in B_{\epsilon}, \nabla L_\alpha^\top \groupgrad_i - \dispoint^\top \groupgrad_i > 0,  \forall i \in [K] \}$. Assume $A$ is feasible and the disagreement point is $\dispoint=0$ (i.e. the update $\gradupd = 0$, staying at $\intermweight$). The goal is to find a $\gradupd$ that maximizes the product of the deviations of each group's payoff from their disagreement point. Since $u_i$ forms a (shifted) linear approximation at $\intermweight$, we are essentially maximizing the utility of $\loss_\gcoefunknown$ locally. We provide further discussions on the problem setup and assumptions in Appendix~\ref{app:setup}. 

\subsection{Solving the problem}
\label{subsec:nash_sol}

In this section, we will show the NBS is (up to scaling) achieved at $\gradupd = \sum_{i\in [\groupcnt]} \gcoefunknown_i \groupgrad_i$, where $\gcoefunknown \in \mathbb{R}_+^K$ solves $\allgrad^\top G \gcoefunknown = \frac{1}{\gcoefunknown}$ by the following two theorems, with full proofs available in Appendix~\ref{app:proof}:
\begin{theorem} \label{thm:main_derivative}
     Under $\dispoint=0$, $\arg\max_{\gradupd \in \alternative} \prod_{i\in[\groupcnt]} (u_i({\gradupd}) - d_i)$ is achieved at
    \begin{align}
    \sum_{i\in[\groupcnt]} \frac{1}{{\gradupd}^\top \groupgrad_i} g_i &= \const \gradupd, \quad \text{for some } \const > 0. \label{eq:solve_for_upd}
    \end{align}
\end{theorem}

\textbf{Proof Sketch.} We employ the same techniques as in Claim 3.1 of \cite{navon2022multi}.

\begin{theorem} \label{thm:solving_grad}
The solution to Equation \ref{eq:solve_for_upd} is (up to scaling) $\gradupd = \sum_{i\in \groupcnt} \gcoefunknown_i \groupgrad_i$ where 
\begin{align} \label{eq:coef_solution}
 \allgrad^\top \allgrad \gcoefunknown = \frac{1}{\gcoefunknown} 
\end{align}
with the element-wise reciprocal $\frac{1}{\gcoefunknown}$. 
\end{theorem}

\textbf{Proof sketch.} Let $x=\gradupd$. In line with \cite{navon2022multi} we observe that $x= \frac{1}{\const} \sum_{i \in \groupcnt} (x^\top \groupgrad_i)^{-1} \groupgrad_i$. However, whereas \cite{navon2022multi} relied on the linear independence of the $g_i$'s to uniquely determine each coefficient $(x^\top \groupgrad_i)^{-1}$, our technique makes no such assumption.  Instead, we multiply both sides of Equation~\ref{eq:solve_for_upd} by $g_j$ and obtain $\sum_{i \in [\groupcnt]} (x^\top \groupgrad_i)^{-1}  (\groupgrad_i^\top \groupgrad_j) = \const x^\top \groupgrad_i, j\in [\groupcnt]$.  
Set $\gcoefunknown_i = (\unknown^\top \groupgrad_i)^{-1}$. This is equivalent to $\groupgrad_j^\top \sum_{i\in [\groupcnt]} \groupgrad_i \gcoefunknown_i = \gcoefunknown_j^{-1}$, which is the desired solution when written in matrix form.

While our solution aligns with that of multi-task learning \cite{navon2022multi}, our proof of Theorem~\ref{thm:solving_grad} circumvents the necessity for linear independence among $\groupgrad_i$, one of the core initial assumptions in the previous work. Linear independence does not hold in general as the goals for individual groups (or tasks) might overlap (such as sharing common underlying features) or contradict each other (when there is negative multiplicity).
Our proof removes this assumption and sheds light on the effectiveness of updates based on the NBS in general cases. See Appendix~\ref{app:proof} for extended discussions on linear independence.

Furthermore, we derive two useful properties of the NBS in additional to its four axioms, with full proofs available in Appendix \ref{app:proof}:
\begin{corollary}\label{corollary:l2norm}
(Norm of bargained update) The solution in Theorem \ref{thm:solving_grad} has $\ell^2$-norm $\sqrt{\groupcnt}$. 
\end{corollary}
\begin{corollary}\label{corollary:alpha_bdd}
If $\groupgrad_j$ is $\bdd$-bounded for $j \in [\groupcnt]$, $ \lVert \gcoefunknown_j^{-1} \rVert$ is $(\sqrt{\groupcnt}\bdd)$-bounded for $j\in [\groupcnt]$. 
\end{corollary}

Informally, the NBS has implicit $\ell^2$ regularization (Corollary \ref{corollary:l2norm}) which substantiates our empirical observation that a separate $\ell^2$-normalization on $\gradupd$ for meta-learning rate adjustment yields better performance than $\ell^1$-normalization (the conventional setting in \cite{ren2018learning}). 
Note that we do not impose any assumptions on the boundedness of the hypergradient $\groupgrad_i$, ensuring the stability even when certain hypergradients are extreme. Furthermore, when $\groupgrad_i$ is bounded, Corollary \ref{corollary:alpha_bdd} implies that $\lVert \gcoefunknown_j \rVert$ is bounded below and no groups are left behind.

\subsection{Game-theoretic underpinnings}
\label{subsec:nash_sol_theory}

The NBS provides incentives for each player to participate in the bargaining game by assuming the existence of at least one feasible solution that all players prefer over disagreement ($\exists x \in \feasresult$ s.t. $x \succ d$). This aligns players' interests in reaching an agreement. The constraint $\groupgrad_i^\top \gradupd > 0$ resolves hypergradient conflict upon agreement.

Second, note that Equation~\ref{eq:coef_solution} shows relationship between the individual and interactive components:
\begin{align}
 \lVert \gcoefunknown_i \groupgrad_i \rVert_2^2 +  \sum_{j \neq i} (\gcoefunknown_i \groupgrad_i)^\top (\gcoefunknown_j \groupgrad_j) = 1, \label{eq:eq_decom_coef_layout}
\end{align}
for $i  \in [\groupcnt]$.
The relative weights $\gcoefunknown_i$ emerge from a player's own impact ($\lVert \gcoefunknown_i \groupgrad_i \rVert_2^2$) and interactions with others ($(\gcoefunknown_i \groupgrad_i)^\top (\gcoefunknown_j \groupgrad_j)$). This trade-off embodies individual versus collective rationality. Positive interactions (i.e., $\groupgrad_i^\top \groupgrad_j > 0$) incentivize collective improvements by downweighting $\gcoefunknown_i$. Negative interactions (i.e., $\groupgrad_i^\top \groupgrad_j < 0$) increase $\gcoefunknown_i$  to prioritize individual objectives. 
Furthermore, each player accounts for a nontrivial contribution to the chosen alternative $\gradupd$ under mild assumptions (Corollary \ref{corollary:alpha_bdd}). The negotiated solution balances individual and collective rationality through participation incentives and conflict resolution. This equilibrium encapsulates game-theoretic bargaining.


\subsection{Two-stage Nash-Meta-Learning}
\label{subsec:nash_meta}


\begin{algorithm}[tb]
\SetKwInput{KwParam}{Parameters}
\small
    \caption{Two-stage Nash-Meta-Learning Training}
    \label{alg:training}
    \SetNoFillComment
    \KwIn{
    \quad \quad \ \  $\param^{(0)}$; $\traindata$; $\valdata = \{\valdata_1,\ldots,\valdata_\groupcnt\}$ 
    }

    \KwParam{
    $\fairnessaggweight_0$ (fairness protocol); $\lr^{(t)}$ (step size); \\
    \quad \quad \quad \ \ \ \quad \quad $T_{bar}$ (bargaining steps)

    }
    \BlankLine
    \For{step $t \in \{1,\ldots,T\}$}{
        Minibatch $D^{(t)}$ sampled from $\traindata$\;
 
        \tcc{\colorbox[HTML]{ececec}{1. Unrolling Inner Optimization}}
                
        $\nabla\param^{(t)}(\intermweight) \gets \backwardad_{\param}\big (\intermweight \cdot L(D^{(t)} | \param^{(t)})\big )$\;
        
        $\hat{\param}^{(t)}(\intermweight) \gets \param^{(t)} - \lr^{(t)} \nabla\param^{(t)}(\intermweight)$\;

        $\fairnessaggweight \gets \fairnessaggweight_0$\;

        \If{$t < T_{bar}$}{
            \tcc{\colorbox[HTML]{ececec}{2. Validation Group-Utilities}}
            \For{group $k \in \{1,\ldots,\groupcnt\}$}{
                $\groupgrad_k^{(t)} \gets \backwardad_{\intermweight}\big (\groupvalloss_k(\valdata | \hat{\param}^{(t)}(\intermweight))\big )$\;}

                \tcc{\colorbox[HTML]{ececec}{3. Bargaining Game}}
                Set $\allgrad^{(t)}$ with columns $\groupgrad_k^{(t)}$\;
                
                Solve for $\gcoefunknown$: $(\allgrad^{(t)})^T \allgrad^{(t)}\gcoefunknown = \frac{1}{\gcoefunknown}$ to obtain $\gcoefunknown^{(t)}$\;
                
                \If{$\gcoefunknown^{(t)}$ is not None}{
                   \tcp{\colorbox[HTML]{d9ead3}{bargaining succeeded}}
                    $\fairnessaggweight \gets \gcoefunknown^{(t)}$\;
                }
            
        }
        $\intermweight \gets \backwardad_{\intermweight}\big (\fairnessloss(\valdata|\hat{\param}^{(t)}(\intermweight))\big )$\;
        
        \tcc{\colorbox[HTML]{ececec}{4. Weighted Update}}
        $ w \gets \normalize \big(\max(-\intermweight, 0)\big)$\;
        
        $\nabla\param^{(t)} \gets \backwardad_{\param} \big (w \cdot L(D^{(t)}|\param^{(t)}) \big)$\;
        
        $\param^{(t+1)} \gets \param^{(t)} - \lr^{(t)} \nabla\param^{(t)}$\;
    }
          
          
          
\KwOut{$\param^{(T)}$
}
\end{algorithm}

We now present our two-stage method (Algorithm \ref{alg:training})
that incorporates Nash bargaining into meta-learning training. Previous one-stage algorithms fix a predetermined $\fairnessloss$; our two-stage method assigns $\fairnessaggweight=\gcoefunknown$ in the NBS in Stage 1 and sets it back to the original $\fairnessaggweight_0$ in Stage 2. Define $\param^{(t)}$ as the model parameter at step $t$, $T$ as the number of total steps, and $T_{(bar)} \leq T$ as the number of bargaining steps in Stage 1. Let $\backwardad_{\param}(L)$ be the backward automatic differentiation of computational graph $L$ w.r.t.\ $\param$, and $\normalize(\cdot)$ be $\ell^2$-normalization function. 
Each step $t$ is constituted by four parts:
\textbf{The first part} is unrolling inner optimization, a common technique to approximate the solution to a nested optimization problem \cite{grefenstette2019generalized}.  We compute an temporary (unweighted) update $\hat{\param}^{(t)}$ on training data, which will be withdrawn after obtaining the udpated $\reweight$. $\intermweight$ is initialized to zero and included into the computation graph. \textbf{The second part} calculates the hypergradient $\groupgrad_k^{(t)}$ which ascertains the descent direction of each training data locally on the validation loss surface. \textbf{The third part} is to aggregate the hypergradients as the update direction for $\intermweight$ by $\gcoefunknown$ from the NBS. In Stage 1, successful bargaining grant the update of $\intermweight$ by the NBS. If we are in Stage 2 (not in the bargaining steps) or the players fail to bargain, we calculate $\intermweight$ based on the fairness loss $\fairnessloss$ of our choice. 
\textbf{The last part} is the update of parameter $\param^{(t+1)}$ using the clipped and normalized weights $\reweight$ for each training data in the minibatch. 

\subsection{Theoretical Properties}
\label{subsec:nash_meta_theory}
\begin{theorem}  \label{thm:theta_upd_rule}
    (Update rule of $\param$) Denote $\trainlosstmp_i = \loss(\minibatch_i^{(t)}|\param^{(t)}) \in \mathbb{R}$ for the $i$-th sample in training minibatch $\minibatch^{(t)}$ at step $t$. $\param$ is updated as 
 $
        \param^{(t+1)} = \param^{(t)} -  \frac{\lr^{(t)}}{|\minibatch^{(t)}|} \sum_{i =1}^{|\minibatch^{(t)}|} \Delta \param^{i}
 $
 with $\Delta \param^{i} = \max{((\nabla_\param ((\fairnessaggweight^{(t)})^\top \vallosstmp))^\top \nabla_\param \trainlosstmp_i, 0)} \nabla_\param \trainlosstmp_i$.
\end{theorem}

\begin{theorem} \label{thm:pareto_improvement}
    (Pareto improvement of $\intermweight$) Use $\alpha^{(t)}$ for the update. Assume $\groupvalloss_i$ is Lipschitz-smooth with constant $\lsmooth$ and $\groupgrad_i^{(t)}$ is $\bdd$-bounded at step $t$. If the meta learning rate for $\intermweight$ satisfies $\metalr^{(t)} \leq \frac{2}{\lsmooth \groupcnt \gcoefunknown_j^{(t)}}$ for $j\in [\groupcnt]$, then $\groupvalloss_i(\intermweight^{(t+1)})  \leq \groupvalloss_i(\intermweight^{(t)})$ for any group $i \in [\groupcnt]$.
\end{theorem}

\begin{theorem} \label{thm:monotone_val}
     Assume $\vallosstmp$ is Lipschitz-smooth with constant $\lsmooth$ and $\nabla_{\param} \trainlosstmp_i$ is $\bdd$-bounded. If the learning rate for $\param$ satisfies $\lr^{(t)} \leq \frac{2 |\minibatch^{(t)}| }{ \lsmooth \lVert \fairnessaggweight^{(t)} \rVert \bdd^2}$, then $ \loss_{\fairnessaggweight^{(t)}} (\param^{(t + 1)}) \leq \loss_{\fairnessaggweight^{(t)}} (\param^{(t)})$ for any fixed vector $\fairnessaggweight^{(t)}$ with finite $\lVert \fairnessaggweight^{(t)} \rVert$ used to update $\param^{(t)}$. 
\end{theorem}
Informally, the closed-form update rule of $\param$ indicates that the weight of a training sample is determined by its local similarity between the $\beta$-reweighted validation loss surface and the training loss surface (Theorem \ref{thm:theta_upd_rule}). Under mild conditions, $\intermweight$ yields Pareto improvement for all groups for the outer optimization using NBS (Theorem \ref{thm:pareto_improvement}). For the inner optimization, under mild conditions, the fairness loss $\loss_{\fairnessaggweight^{(t)}}$ monotonically decreases w.r.t.\ $\fairnessaggweight^{(t)}$ regardless of the choice of protocol (Theorem \ref{thm:monotone_val}). This generalizes \cite{ren2018learning} from $\fairnessaggweight = \frac{1}{\groupcnt} \mathbbm{1}$ to any $\fairnessaggweight$ with finite norm and provides a uniform property for the fairness-aware meta-learning methods. It entails the flexibility of our two-stage design that switches $\fairnessaggweight$ between phases. Setting $\fairnessaggweight=\gcoefunknown$ gives the desired property for the NBS. The validation loss surface reweighted by the NBS has the maximum joint utility, which empirically boosts the overall performance when used to update $\param$. Full proofs are given in Appendix \ref{app:proof}.


\subsection{Dynamics of Two-stage Nash-Meta-Learning}
\label{subsec:nash_meta_dynamics}
Our method captures the interplay and synergy among different groups. Specifically, previous methods like linear scalarization (i.e., assigning a fix weight to each group) are limited to identifying points on the convex envelope of the Pareto front \cite{boyd2004convex}. Our method offers a more adaptable navigation mechanism by dynamically adjusting the weight with the NBS, which accounts for the intricate interactions and negotiations among groups. Moreover, the optimal $\gradupd$ maximizes the utilization of information on validation loss landscape and leads to empirical faster and more robust convergence to the Pareto front even with distant initial points. Although first-order methods typically avoid saddle points \cite{frank1956algorithm, panageas2019first}, if one is encountered, switching to the fairness goal upon unsuccessful bargaining offers a fresh starting point for subsequent bargaining iterations and helps to escape (Figure \ref{fig:alignment}, Appendix \ref{app:additional_result}). Our experiments show that Stage 2 training focused solely on the fairness goal does not deviate the model from the Pareto front (Figure \ref{fig:oursimu}). Specifically, the worst group utility $\groupgrad_i\groupgradagg$ tends to concentrate around zero, ensuring the model stays in the neighborhood of the Pareto front and implying the robustness of our approach. The theoretical understanding of this interesting phenomenon is an open problem for future study.

\section{Evaluation}
\label{sec:evaluation}
We evaluate our method in three key areas: synthetic simulation ($\S$\ref{subsec:pareto}) for Pareto optimality vis-à-vis fairness objectives, real-world fairness datasets ($\S$\ref{subsec:fairness}), and two imbalanced image classification scenarios (Appendix \ref{subsec:image}).

\subsection{Simulation}
\label{subsec:pareto}

Under the synthetic settings ($\S$\ref{sec:problem} and Appendix \ref{app:exp_settings}), we observe the convergence enhancements compared to Figure \ref{fig:problem} and the effect of continuous bargaining throughout the entire training (another case of one-stage).
%
Figure \ref{fig:oursimu} demonstrates that Nash bargaining effectively resolves gradient conflicts and facilitates convergence to the Pareto front in comparison to Figure \ref{fig:problem}. This is evident from the reduced number of non-Pareto-converged nodes in both continuous bargaining (Figure \ref{fig:oursimu}\textcolor{blue}{a}) and early-stage bargaining (the final solution, Figure \ref{fig:oursimu}{\textcolor{blue}{b}). Notice that one-stage NBS doesn't always enhance fairness, evidenced by the observation that nodes at the Pareto front tend to stagnate. The NBS does not leverage any information about specific fairness objectives. Our two-stage approach built on the bargaining steps can further push the model to the ideal endpoints. 

\begin{figure}[!h] 
    \centering
    \includegraphics[width=\linewidth]{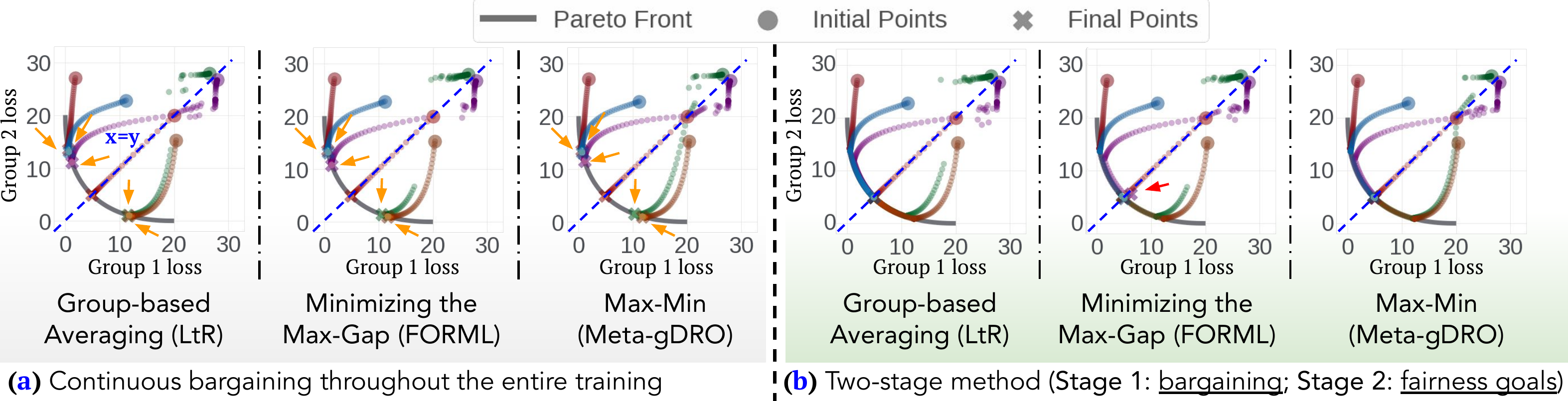}
    \caption{Synthetic illustration of the bargaining's effects. ``\textcolor{myorange}{\ding{220}}'': final point not close to the fairness goal (\textcolor{blue}{\textbf{x=y}}). ``\textcolor{myred}{\ding{220}}'': final point not at the Pareto front. \textbf{(\textcolor{blue}{a})} Bargaining across all 1000 steps; \textbf{(\textcolor{blue}{b})} Bargaining only included in the first 100 steps (two-stage method). 
    }
    \label{fig:oursimu}
\end{figure}

\begin{table*}[h!]
\centering
\resizebox{\linewidth}{!}{
\begin{tabular}{@{}cc|ccccccc@{}}
\toprule
\multicolumn{1}{c|}{\multirow{2}{*}{}}        & \multicolumn{1}{C{6.em}|}{\multirow{2}{*}{\textbf{Baseline}}} & \multicolumn{1}{C{6.em}|}{\multirow{2}{*}{\textbf{DRO}}} & \multicolumn{2}{c|}{\cellcolor[HTML]{e6ddd6}\textbf{LtR}}           & \multicolumn{2}{c|}{\cellcolor[HTML]{ecdcf9}\textbf{FORML}}         & \multicolumn{2}{c}{\cellcolor[HTML]{e7effe}\textbf{Meta-gDRO}} \\ \cmidrule(l){4-9} 
\multicolumn{1}{c|}{}                         & \multicolumn{1}{c|}{}                                   & \multicolumn{1}{c|}{}                              & \multicolumn{1}{C{6.em}}{one-stage}   & \multicolumn{1}{c|}{two-stage (ours)}        & \multicolumn{1}{C{5.em}}{one-stage}   & \multicolumn{1}{c|}{two-stage (ours)}        & \multicolumn{1}{C{6.em}}{one-stage}           & two-stage (ours)              \\

\hline \midrule
\multicolumn{9}{l}{\multirow{2}{*}{\textbf{I. Adult Income}~\cite{misc_adult_2}, (sensitive attribute: \textbf{sex})}}\\ \\  \midrule 

\multicolumn{1}{r|}{\textbf{Overall AUC (↑)}} & \multicolumn{1}{c|}{0.778}                              & \multicolumn{1}{c|}{0.761}                         & \cellcolor[HTML]{e6ddd6}0.830 & \multicolumn{1}{c|}{\cellcolor[HTML]{e6ddd6}0.830 (+.000)} & 0.801 & \multicolumn{1}{c|}{0.810 (+.009)} & 0.810         & \textbf{0.837} (+.027)         \\

\multicolumn{1}{r|}{\textbf{Max-gAUCD (↓)}}   & \multicolumn{1}{c|}{\textbf{0.016}}                              & \multicolumn{1}{c|}{0.029}                         & 0.050 & \multicolumn{1}{c|}{0.047 (-.003)}  & \cellcolor[HTML]{ecdcf9}0.075 & \multicolumn{1}{c|}{\cellcolor[HTML]{ecdcf9}0.031 (-.044)} & 0.052        & 0.046 (-.006)         \\

\multicolumn{1}{r|}{\textbf{Worst-gAUC (↑)}}  & \multicolumn{1}{c|}{0.770}                              & \multicolumn{1}{c|}{0.747}                         & 0.805 & \multicolumn{1}{c|}{0.807 (+.002)} & 0.763 & \multicolumn{1}{c|}{0.795 (+.032)} & \cellcolor[HTML]{e7effe}0.809         & \cellcolor[HTML]{e7effe}\textbf{0.814} (+.006)   \\


\hline \midrule
\multicolumn{9}{l}{\multirow{2}{*}{\textbf{II. Adult Income}~\cite{misc_adult_2}, (sensitive attribute: \textbf{race};
\scalebox{1}{\colorbox[HTML]{FFEEC2}{* one group in training data contains only one positive (favorable) label.}}
)}} \\ \\  \midrule 

\multicolumn{1}{r|}{\textbf{Overall AUC (↑)}} & \multicolumn{1}{c|}{0.668}                              & \multicolumn{1}{c|}{0.652}                         & \cellcolor[HTML]{e6ddd6}0.803 & \multicolumn{1}{c|}{\cellcolor[HTML]{e6ddd6}\textbf{0.805} (+.002)} & 0.710 & \multicolumn{1}{c|}{0.775 (+.065)} & 0.775         & 0.793 (+.018)         \\

\multicolumn{1}{r|}{\textbf{Max-gAUCD (↓)}}   & \multicolumn{1}{c|}{0.225}                              & \multicolumn{1}{c|}{0.236}                         & \textbf{0.090} & \multicolumn{1}{c|}{\textbf{0.090} (-.000)}  & \cellcolor[HTML]{ecdcf9}0.290 & \multicolumn{1}{c|}{\cellcolor[HTML]{ecdcf9}0.134 (-.156)} & 0.163         & 0.158 (-.005)         \\

\multicolumn{1}{r|}{\textbf{Worst-gAUC (↑)}}  & \multicolumn{1}{c|}{0.544}                              & \multicolumn{1}{c|}{0.538}                         & 0.755 & \multicolumn{1}{c|}{\textbf{0.760} (+.005)} & 0.540 & \multicolumn{1}{c|}{0.688 (+.148)} & \cellcolor[HTML]{e7effe}0.694         & \cellcolor[HTML]{e7effe}0.703 (+.009)   \\

\hline \midrule
\multicolumn{9}{l}{\multirow{2}{*}{\textbf{III. Bank Telemarketing}~\cite{moro2014data}, (sensitive attribute: \textbf{age})}}\\ \\  \midrule 

\multicolumn{1}{r|}{\textbf{Overall AUC (↑)}} & \multicolumn{1}{c|}{0.697}                              & \multicolumn{1}{c|}{0.686}                         & \cellcolor[HTML]{e6ddd6}0.724 & \multicolumn{1}{c|}{\cellcolor[HTML]{e6ddd6}0.728 (+.004)} & 0.706 & \multicolumn{1}{c|}{\textbf{0.779} (+.073)} & 0.698         & 0.722 (+.024)         \\

\multicolumn{1}{r|}{\textbf{Max-gAUCD (↓)}}   & \multicolumn{1}{c|}{\textbf{0.013}}                              & \multicolumn{1}{c|}{0.025}                         & 0.099 & \multicolumn{1}{c|}{0.083 (-.016)}  & \cellcolor[HTML]{ecdcf9}0.039 & \multicolumn{1}{c|}{\cellcolor[HTML]{ecdcf9}0.029 (-.010)} & 0.098         & 0.079 (-.019)         \\

\multicolumn{1}{r|}{\textbf{Worst-gAUC (↑)}}  & \multicolumn{1}{c|}{0.691}                              & \multicolumn{1}{c|}{0.691}                         & 0.675 & \multicolumn{1}{c|}{0.686 (+.011)} & 0.686 & \multicolumn{1}{c|}{\textbf{0.764} (+.078)} & \cellcolor[HTML]{e7effe}0.649         & \cellcolor[HTML]{e7effe}0.683 (+.034)   \\


\hline \midrule
\multicolumn{9}{l}{\multirow{2}{*}{\textbf{IV. Credit Default}~\cite{yeh2009comparisons}, (sensitive attribute: \textbf{sex};
\scalebox{1}{\colorbox[HTML]{FFEEC2}{* val data is more noisy than others, see analysis in Table \ref{tab:noisy_test_set}, Appendix \ref{app:additional_result}.}}
)}}\\ \\  \midrule 

\multicolumn{1}{r|}{\textbf{Overall AUC (↑)}} & \multicolumn{1}{c|}{0.634}                              & \multicolumn{1}{c|}{0.624}                         & \cellcolor[HTML]{FFFFFF}0.630 & \multicolumn{1}{c|}{\cellcolor[HTML]{FFFFFF}0.616 (-.014)} & 0.554 & \multicolumn{1}{c|}{0.611 (+.057)} & \textbf{0.682}        & 0.661 (-.021)         \\

\multicolumn{1}{r|}{\textbf{Max-gAUCD (↓)}}   & \multicolumn{1}{c|}{0.024}                              & \multicolumn{1}{c|}{0.022}                         & 0.037 & \multicolumn{1}{c|}{0.025 (-.012)}  & \cellcolor[HTML]{FFFFFF}0.017 & \multicolumn{1}{c|}{\cellcolor[HTML]{FFFFFF}0.033 (+.016)} & \textbf{0.016}         & 0.022 (+.006)         \\

\multicolumn{1}{r|}{\textbf{Worst-gAUC (↑)}}  & \multicolumn{1}{c|}{0.622}                              & \multicolumn{1}{c|}{0.613}                         & 0.612 & \multicolumn{1}{c|}{0.603 (-.009)} & 0.545 & \multicolumn{1}{c|}{0.595 (+.050)} & \cellcolor[HTML]{FFFFFF}\textbf{0.674}         & \cellcolor[HTML]{FFFFFF}0.650 (-.024)   \\


\hline \midrule
\multicolumn{9}{l}{\multirow{2}{*}{\textbf{V. Communities and Crime}~\cite{misc_communities_and_crime_183}, (sensitive attribute: \textbf{blackgt6pct}; 
\scalebox{1}{\colorbox[HTML]{FFEEC2}{* val data is noisy, meanwhile one group in training data are all positive labels}}
)}}\\ \\  \midrule 

\multicolumn{1}{r|}{\textbf{Overall AUC (↑)}} & \multicolumn{1}{c|}{0.525}                              & \multicolumn{1}{c|}{0.568}                         & \cellcolor[HTML]{e6ddd6}0.679 & \multicolumn{1}{c|}{\cellcolor[HTML]{e6ddd6}\textbf{0.700} (+.021)} & 0.554 & \multicolumn{1}{c|}{0.568 (+.014)} & 0.686        & 0.686 (+.000)         \\

\multicolumn{1}{r|}{\textbf{Max-gAUCD (↓)}}   & \multicolumn{1}{c|}{\textbf{0.050}}                              & \multicolumn{1}{c|}{0.136}                         & 0.071 & \multicolumn{1}{c|}{0.129 (+.058)}  & \cellcolor[HTML]{FFFFFF}0.107 & \multicolumn{1}{c|}{\cellcolor[HTML]{FFFFFF}0.136 (+.029)} & 0.114         & 0.114 (-.000)         \\

\multicolumn{1}{r|}{\textbf{Worst-gAUC (↑)}}  & \multicolumn{1}{c|}{0.500}                              & \multicolumn{1}{c|}{0.500}                         & \textbf{0.643} & \multicolumn{1}{c|}{0.636 (-.007)} & 0.500 & \multicolumn{1}{c|}{0.500 (+.000)} & \cellcolor[HTML]{e7effe}0.629         & \cellcolor[HTML]{e7effe}0.629 (+.000)   \\


\hline \midrule
\multicolumn{9}{l}{\multirow{2}{*}{\textbf{VI. Titanic Survival}~\cite{titanic}, (sensitive attribute: \textbf{sex})}}\\ \\ \midrule 

\multicolumn{1}{r|}{\textbf{Overall AUC (↑)}} & \multicolumn{1}{c|}{0.972}                              & \multicolumn{1}{c|}{\textbf{0.983}}                         & \cellcolor[HTML]{e6ddd6}0.967 & \multicolumn{1}{c|}{\cellcolor[HTML]{e6ddd6}0.978 (+.011)} & 0.950 & \multicolumn{1}{c|}{0.972 (+.022)} & 0.961        & 0.972 (+.011)         \\

\multicolumn{1}{r|}{\textbf{Max-gAUCD (↓)}}   & \multicolumn{1}{c|}{0.056}                              & \multicolumn{1}{c|}{0.033}                         & 0.044 & \multicolumn{1}{c|}{0.044 (-0.00)}  & \cellcolor[HTML]{ecdcf9}0.033 & \multicolumn{1}{c|}{\cellcolor[HTML]{ecdcf9}\textbf{0.011} (-.022)} & 0.033        & 0.033 (-.000)         \\

\multicolumn{1}{r|}{\textbf{Worst-gAUC (↑)}}  & \multicolumn{1}{c|}{0.944}                              & \multicolumn{1}{c|}{\textbf{0.967}}                         & 0.944 & \multicolumn{1}{c|}{0.956 (+.012)} & 0.933 & \multicolumn{1}{c|}{\textbf{0.967} (+.034)} & \cellcolor[HTML]{e7effe}0.944         & \cellcolor[HTML]{e7effe}0.956 (+.012)   \\


\hline \midrule
\multicolumn{9}{l}{
\multirow{2}{*}{
\textbf{VII. Student Performance}~\cite{misc_student_performance_320}, (sensitive attribute: \textbf{sex})
}}\\ \\ \midrule

\multicolumn{1}{r|}{\textbf{Overall AUC (↑)}} & \multicolumn{1}{c|}{0.784}                              & \multicolumn{1}{c|}{0.816}                         & \cellcolor[HTML]{e6ddd6}0.900 & \multicolumn{1}{c|}{\cellcolor[HTML]{e6ddd6}0.900 (+.000)} & 0.828 & \multicolumn{1}{c|}{0.822 (-.006)} & 0.909         & \textbf{0.912} (+.003)         \\

\multicolumn{1}{r|}{\textbf{Max-gAUCD (↓)}}   & \multicolumn{1}{c|}{0.119}                              & \multicolumn{1}{c|}{0.106}                         & \textbf{0.013} & \multicolumn{1}{c|}{0.037 (+.024)}  & \cellcolor[HTML]{ecdcf9}0.056 & \multicolumn{1}{c|}{\cellcolor[HTML]{ecdcf9}0.031 (-.025)} & 0.031        & 0.025 (-.006)         \\

\multicolumn{1}{r|}{\textbf{Worst-gAUC (↑)}}  & \multicolumn{1}{c|}{0.725}                              & \multicolumn{1}{c|}{0.762}                         & 0.894 & \multicolumn{1}{c|}{0.881 (+.013)} & 0.800 & \multicolumn{1}{c|}{0.806 (+.006)} & \cellcolor[HTML]{e7effe}0.894         & \cellcolor[HTML]{e7effe}\textbf{0.900} (+.006)   \\

\bottomrule
\end{tabular}}
\caption{Comparison on standard fairness datasets (averaged from 5 runs). Each of \{\scalebox{0.8}{\colorbox[HTML]{e6ddd6}{LtR}}, \scalebox{0.8}{\colorbox[HTML]{ecdcf9}{FORML}}, \scalebox{0.8}{\colorbox[HTML]{e7effe}{Meta-gDRO}}\} is paired with \{\scalebox{0.8}{\colorbox[HTML]{e6ddd6}{Overall AUC}}, \scalebox{0.8}{\colorbox[HTML]{ecdcf9}{Max-gAUCD}}, \scalebox{0.8}{\colorbox[HTML]{e7effe}{Worst-gAUC}}\} rsp. for aligning intended fairness goals. Top results of each row in \textbf{bold}.
}
\label{table:fairness}
\end{table*}

\subsection{Standard fairness benchmarks}
\label{subsec:fairness}
We test our method on six standard fairness datasets across various sectors of fairness tasks: financial services (adult income \cite{misc_adult_2}, credit default \cite{yeh2009comparisons}), marketing (bank telemarketing \cite{moro2014data}), criminal justice (communities and crime \cite{misc_communities_and_crime_183}), education (student performance \cite{misc_student_performance_320}), and disaster response (Titanic survival \cite{titanic}). 
%
Test sets comprise 3\% of each dataset (10\% for the student performance dataset with 649 samples) by randomly selecting a demographically and label-balanced subset. See Table \ref{tab:data_summary} in Appendix \ref{app:exp_settings} for data distribution specifics.

\textbf{General settings and metrics.}
We compare our two-stage Nash-Meta-Learning with traditional one-stage fairness-aware meta-learning (i.e., LtR, FORML, Meta-gDRO), baseline training, and Distributional Robust Optimization (DRO) \cite{hashimoto2018fairness}. All methods share the same model architecture and training hyperparameters on each dataset. Our approach features a 15-epoch bargaining phase within the total 50 epochs. See Appendix \ref{app:exp_settings} for training details. Three metrics from $\S$\ref{sec:problem} are used: \scalebox{0.8}{\colorbox[HTML]{e6ddd6}{\textbf{Overall AUC ($\uparrow$)}}}, \scalebox{0.8}{\colorbox[HTML]{ecdcf9}{\textbf{Max-gAUCD ($\downarrow$)}}}, and \scalebox{0.8}{\colorbox[HTML]{e7effe}{\textbf{Worst-gAUC ($\uparrow$)}}}, corresponding to the goal of LtR, FORML, and Meta-gDRO, respectively.

\textbf{Results and analysis.}
Our NBS-enhanced two-stage meta-learning improves the overall performance, fairness objectives (in color), and stability, as in Table \ref{table:fairness} and with 95\% CI in Table \ref{table:fairness_CI}, Appendix \ref{app:additional_result}.  For instance, with our bargaining steps, FORML's Overall AUC rose by 10.34\% (from 0.706 to 0.779), with a tighter 95\% CI (from 0.202 to 0.054), and the Max-gAUCD decreased by 26\% (from 0.039 to 0.029), with a tighter CI (from 0.046 to 0.018). 
However, our method faced challenges in two datasets: Credit Default, where performance and fairness occasionally declined, and Communities and Crime, where minimal improvement was observed (in particular, the Meta-gDRO). We diagnose that these two dataset's validation set contains low feature-label mutual information, leading to noisy outcomes (Appendix \ref{app:additional_result}) and affecting most tested methods (e.g., LtR, FORML). Additionally, for Communities and Crime, our method is influenced by the low bargaining success/feasible rate, possible due to the lack of favorable (positive) training samples for the minority groups (Table \ref{tab:data_summary}, Appendix \ref{app:exp_settings}). Conversely, our method still yields the anticipated bargaining results on the adult income dataset with only one positive \texttt{Amer-Indian} sample. These insights emphasize the importance of validation set quality and representative samples in the training.


\begin{wrapfigure}{r}{0.49\textwidth}
    \centering
    \includegraphics[width=\linewidth]{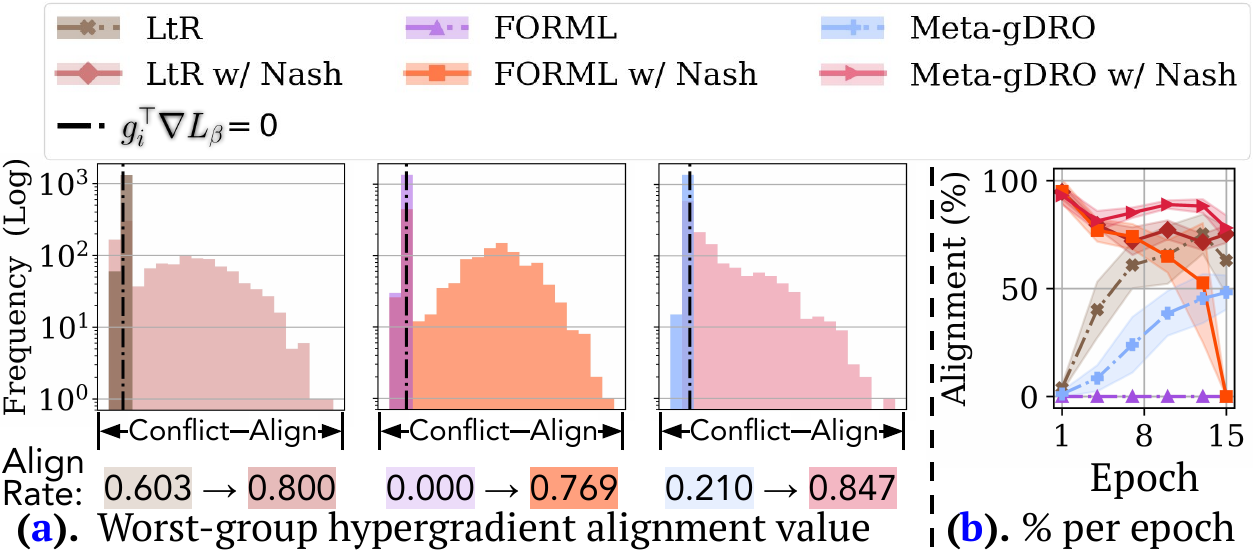}
    \caption{Effects on hypergradient alignment (Bank Telemarketing). \textbf{(\textcolor{blue}{a})} Smallest $\groupgrad_i^\top \nabla \loss_\beta$. Portion of positive values (Align. Rate) at the bottom. \textbf{(\textcolor{blue}{b})} Hypergradient alignment rate per epoch.}
    \label{fig:telemark}
\end{wrapfigure}
\textbf{Effects of bargaining on hypergradient conflicts.}
Bargaining enhances gradient alignment to varying degrees among different one-stage algorithms (Figure \ref{fig:telemark}, \ref{fig:all_uap}). For instance, LtR's alignment rate improves from 60\% to 80\%, and FORML jumps from 0 to 76.9\% accompanied with more substantial performance and fairness gains on Bank Telemarketing (Figure \ref{fig:telemark}). FORML consistently benefits more from bargaining compared to the other two, likely due to its optimization goal that could intensify hypergradient conflicts.
Moreover, our approach uniformly promotes hypergradient alignment during Stage 1 (Figure \ref{fig:all_uap}). We show that early bargaining and hypergradient conflict resolution are crucial for enhancing model performance and fairness. Further illustrations and analyses are in Appendix \ref{app:additional_result}.





\section{Discussion \& Conclusion}
\label{sec:discussion}

In conclusion, our study offers several key insights: We identified \textit{hypergradient conflict} as a pivotal issue undermining stable performance in one-stage fairness-aware meta-learning algorithms. To mitigate this, we proposed a two-stage framework that initially employs the NBS to resolve these conflicts, and then optimizes for fairness. Our assumption-free proof of the NBS extended its applicability to a broader range of gradient aggregation problems. Empirical results demonstrated our method's superior performance across synthetic scenarios, seven real-world fairness settings on six key fairness datasets, and two image classification tasks. 

\textbf{Future directions.}
First, addressing the absence of a specific label in the training subgroup and the low quality of the validation set that affected our method's effectiveness may be mitigated by fairness-aware synthetic data \cite{van2021decaf} or data-sifting methods \cite{zeng2023meta}. Pairing our method with these methods and exploring more resilient solutions adapted to these extreme cases can be a promising direction.
Second, Theorem \ref{thm:monotone_val} establishes the validity to switch the choice of $\fairnessaggweight$ during training. Future work can focus on designing and flexibly choosing outer optimization goals to further improve performance, fairness, or other metrics in interest.
Third, we derive the NBS under $\dispoint=0$ for conflict resolution. Future study could investigate general $\dispoint \neq 0$, which might not be useful for conflict resolution (the scope of our paper), but could be used in other cases as a gradient aggregation method that gains advantages from axiomatic properties, as discussed in Appendix~\ref{app:setup}.

\section*{Acknowledgments}
\label{sec:ack}
We acknowledge Aram H. Markosyan and Vitor Albiero at Meta AI for their insightful discussions.
Ruoxi Jia and the ReDS lab acknowledge the support from the National Science Foundation under Grant No. IIS-2312794.
Michael Jordan wishes to acknowledge support from the European Union (ERC-2022-SYG-OCEAN-101071601).


\bibliographystyle{plain}
\bibliography{allbib,fair}

\begin{thebibliography}{10}

\bibitem{ashurst2023fairness}
Carolyn Ashurst and Adrian Weller.
\newblock Fairness without demographic data: A survey of approaches.
\newblock In {\em Proceedings of the 3rd ACM Conference on Equity and Access in Algorithms, Mechanisms, and Optimization}, pages 1--12, 2023.

\bibitem{baek2021fair}
Jackie Baek and Vivek Farias.
\newblock Fair exploration via axiomatic bargaining.
\newblock {\em Advances in Neural Information Processing Systems}, 34:22034--22045, 2021.

\bibitem{misc_adult_2}
Barry Becker and Ronny Kohavi.
\newblock {Adult}.
\newblock UCI Machine Learning Repository, 1996.
\newblock {DOI}: https://doi.org/10.24432/C5XW20.

\bibitem{binois2020kalai}
Micka{\"e}l Binois, Victor Picheny, Patrick Taillandier, and Abderrahmane Habbal.
\newblock The kalai-smorodinsky solution for many-objective bayesian optimization.
\newblock {\em The Journal of Machine Learning Research}, 21(1):5978--6019, 2020.

\bibitem{boyd2004convex}
Stephen~P Boyd and Lieven Vandenberghe.
\newblock {\em Convex Optimization}.
\newblock Cambridge University Press, 2004.

\bibitem{bullen2013handbook}
Peter~S Bullen.
\newblock {\em Handbook of Means and their Inequalities}, volume 560.
\newblock Springer Science \& Business Media, 2013.

\bibitem{caton2020fairness}
Simon Caton and Christian Haas.
\newblock Fairness in machine learning: A survey.
\newblock {\em ACM Computing Surveys}, 2020.

\bibitem{chai2022self}
Junyi Chai and Xiaoqian Wang.
\newblock Self-supervised fair representation learning without demographics.
\newblock {\em Advances in Neural Information Processing Systems}, 35:27100--27113, 2022.

\bibitem{chang2020multi}
Ho-Chun~Herbert Chang.
\newblock Multi-issue bargaining with deep reinforcement learning.
\newblock {\em arXiv:2002.07788}, 2020.

\bibitem{chen2023algorithmic}
Richard~J Chen, Judy~J Wang, Drew~FK Williamson, Tiffany~Y Chen, Jana Lipkova, Ming~Y Lu, Sharifa Sahai, and Faisal Mahmood.
\newblock Algorithmic fairness in artificial intelligence for medicine and healthcare.
\newblock {\em Nature Biomedical Engineering}, 7(6):719--742, 2023.

\bibitem{misc_student_performance_320}
Paulo Cortez.
\newblock {Student Performance}.
\newblock UCI Machine Learning Repository, 2014.
\newblock {DOI}: https://doi.org/10.24432/C5TG7T.

\bibitem{titanic}
Will Cukierski.
\newblock Titanic - machine learning from disaster, 2012.

\bibitem{datta2017proxy}
Anupam Datta, Matt Fredrikson, Gihyuk Ko, Piotr Mardziel, and Shayak Sen.
\newblock Proxy non-discrimination in data-driven systems.
\newblock {\em arXiv:1707.08120}, 2017.

\bibitem{diaz2023connecting}
Natalia D{\'\i}az-Rodr{\'\i}guez, Javier Del~Ser, Mark Coeckelbergh, Marcos~L{\'o}pez de~Prado, Enrique Herrera-Viedma, and Francisco Herrera.
\newblock Connecting the dots in trustworthy artificial intelligence: From ai principles, ethics, and key requirements to responsible ai systems and regulation.
\newblock {\em Information Fusion}, page 101896, 2023.

\bibitem{forgo1999introduction}
Ferenc Forg{\'o}, Jen{\"o} Sz{\'e}p, and Ferenc Szidarovszky.
\newblock {\em Introduction to the Theory of Games: Concepts, Methods, Applications}, volume~32.
\newblock Springer Science \& Business Media, 1999.

\bibitem{frank1956algorithm}
Marguerite Frank, Philip Wolfe, et~al.
\newblock An algorithm for quadratic programming.
\newblock {\em Naval Research Logistics Quarterly}, 3(1-2):95--110, 1956.

\bibitem{grefenstette2019generalized}
Edward Grefenstette, Brandon Amos, Denis Yarats, Phu~Mon Htut, Artem Molchanov, Franziska Meier, Douwe Kiela, Kyunghyun Cho, and Soumith Chintala.
\newblock Generalized inner loop meta-learning.
\newblock {\em arXiv:1910.01727}, 2019.

\bibitem{hashimoto2018fairness}
Tatsunori Hashimoto, Megha Srivastava, Hongseok Namkoong, and Percy Liang.
\newblock Fairness without demographics in repeated loss minimization.
\newblock In {\em International Conference on Machine Learning}, pages 1929--1938. PMLR, 2018.

\bibitem{he2016deep}
Kaiming He, Xiangyu Zhang, Shaoqing Ren, and Jian Sun.
\newblock Deep residual learning for image recognition.
\newblock In {\em Proceedings of the IEEE Computer Vision and Pattern Recognition}, pages 770--778, 2016.

\bibitem{hu2024revisiting}
Yuzheng Hu, Ruicheng Xian, Qilong Wu, Qiuling Fan, Lang Yin, and Han Zhao.
\newblock Revisiting scalarization in multi-task learning: A theoretical perspective.
\newblock {\em Advances in Neural Information Processing Systems}, 36, 2024.

\bibitem{iosifidis2019fae}
Vasileios Iosifidis, Besnik Fetahu, and Eirini Ntoutsi.
\newblock Fae: A fairness-aware ensemble framework.
\newblock In {\em 2019 IEEE International Conference on Big Data (Big Data)}, pages 1375--1380. IEEE, 2019.

\bibitem{kamiran2012data}
Faisal Kamiran and Toon Calders.
\newblock Data preprocessing techniques for classification without discrimination.
\newblock {\em Knowledge and Information Systems}, 33(1):1--33, 2012.

\bibitem{keshanian2022features}
Kimia Keshanian, Daniel Zantedeschi, and Kaushik Dutta.
\newblock Features selection as a nash-bargaining solution: Applications in online advertising and information systems.
\newblock {\em INFORMS Journal on Computing}, 34(5):2485--2501, 2022.

\bibitem{khan2007relative}
Shiraj Khan, Sharba Bandyopadhyay, Auroop~R Ganguly, Sunil Saigal, David~J Erickson~III, Vladimir Protopopescu, and George Ostrouchov.
\newblock Relative performance of mutual information estimation methods for quantifying the dependence among short and noisy data.
\newblock {\em Physical Review E}, 76(2):026209, 2007.

\bibitem{krizhevsky2009learning}
Alex Krizhevsky, Geoffrey Hinton, et~al.
\newblock Learning multiple layers of features from tiny images.
\newblock 2009.

\bibitem{kum2022pseudo}
Sangeun Kum, Jongpil Lee, Keunhyoung~Luke Kim, Taehyoung Kim, and Juhan Nam.
\newblock Pseudo-label transfer from frame-level to note-level in a teacher-student framework for singing transcription from polyphonic music.
\newblock In {\em ICASSP 2022-2022 IEEE International Conference on Acoustics, Speech and Signal Processing (ICASSP)}, pages 796--800. IEEE, 2022.

\bibitem{lahoti2020fairness}
Preethi Lahoti, Alex Beutel, Jilin Chen, Kang Lee, Flavien Prost, Nithum Thain, Xuezhi Wang, and Ed~Chi.
\newblock Fairness without demographics through adversarially reweighted learning.
\newblock {\em Advances in Neural Information Processing Systems}, 33:728--740, 2020.

\bibitem{leslie2020understanding}
David Leslie.
\newblock Understanding bias in facial recognition technologies.
\newblock {\em arXiv:2010.07023}, 2020.

\bibitem{liu2021conflict}
Bo~Liu, Xingchao Liu, Xiaojie Jin, Peter Stone, and Qiang Liu.
\newblock Conflict-averse gradient descent for multi-task learning.
\newblock {\em Advances in Neural Information Processing Systems}, 34:18878--18890, 2021.

\bibitem{luong2011k}
Binh~Thanh Luong, Salvatore Ruggieri, and Franco Turini.
\newblock K-nn as an implementation of situation testing for discrimination discovery and prevention.
\newblock In {\em Proceedings of the 17th ACM SIGKDD International Conference on Knowledge Discovery and Data Mining}, pages 502--510, 2011.

\bibitem{mehrabi2021survey}
Ninareh Mehrabi, Fred Morstatter, Nripsuta Saxena, Kristina Lerman, and Aram Galstyan.
\newblock A survey on bias and fairness in machine learning.
\newblock {\em ACM Computing Surveys (CSUR)}, 54(6):1--35, 2021.

\bibitem{mitchell2021algorithmic}
Shira Mitchell, Eric Potash, Solon Barocas, Alexander D'Amour, and Kristian Lum.
\newblock Algorithmic fairness: Choices, assumptions, and definitions.
\newblock {\em Annual Review of Statistics and its Application}, 8:141--163, 2021.

\bibitem{moro2014data}
S{\'e}rgio Moro, Paulo Cortez, and Paulo Rita.
\newblock A data-driven approach to predict the success of bank telemarketing.
\newblock {\em Decision Support Systems}, 62:22--31, 2014.

\bibitem{nash1953two}
John Nash.
\newblock Two-person cooperative games.
\newblock {\em Econometrica: Journal of the Econometric Society}, pages 128--140, 1953.

\bibitem{nash1950bargaining}
John~F Nash~Jr.
\newblock The bargaining problem.
\newblock {\em Econometrica: Journal of the Econometric Society}, pages 155--162, 1950.

\bibitem{nash1950equilibrium}
John~F Nash~Jr.
\newblock Equilibrium points in n-person games.
\newblock {\em Proceedings of the National Academy of Sciences}, 36(1):48--49, 1950.

\bibitem{navon2022multi}
Aviv Navon, Aviv Shamsian, Idan Achituve, Haggai Maron, Kenji Kawaguchi, Gal Chechik, and Ethan Fetaya.
\newblock Multi-task learning as a bargaining game.
\newblock {\em arXiv:2202.01017}, 2022.

\bibitem{pagano2023bias}
Tiago~P Pagano, Rafael~B Loureiro, Fernanda~VN Lisboa, Rodrigo~M Peixoto, Guilherme~AS Guimar{\~a}es, Gustavo~OR Cruz, Maira~M Araujo, Lucas~L Santos, Marco~AS Cruz, Ewerton~LS Oliveira, et~al.
\newblock Bias and unfairness in machine learning models: a systematic review on datasets, tools, fairness metrics, and identification and mitigation methods.
\newblock {\em Big Data and Cognitive Computing}, 7(1):15, 2023.

\bibitem{panageas2019first}
Ioannis Panageas, Georgios Piliouras, and Xiao Wang.
\newblock First-order methods almost always avoid saddle points: The case of vanishing step-sizes.
\newblock {\em Advances in Neural Information Processing Systems}, 32, 2019.

\bibitem{qiao2006multi}
Haiyan Qiao, Jerzy Rozenblit, Ferenc Szidarovszky, and Lizhi Yang.
\newblock Multi-agent learning model with bargaining.
\newblock In {\em Proceedings of the 2006 Winter Simulation Conference}, pages 934--940. IEEE, 2006.

\bibitem{rawls2001justice}
John Rawls.
\newblock {\em Justice as Fairness: A Restatement}.
\newblock Harvard University Press, 2001.

\bibitem{misc_communities_and_crime_183}
Michael Redmond.
\newblock {Communities and Crime}.
\newblock UCI Machine Learning Repository, 2009.
\newblock {DOI}: https://doi.org/10.24432/C53W3X.

\bibitem{ren2018learning}
Mengye Ren, Wenyuan Zeng, Bin Yang, and Raquel Urtasun.
\newblock Learning to reweight examples for robust deep learning.
\newblock In {\em International Conference on Machine Learning}, pages 4334--4343. PMLR, 2018.

\bibitem{sagawa2019distributionally}
Shiori Sagawa, Pang~Wei Koh, Tatsunori~B Hashimoto, and Percy Liang.
\newblock Distributionally robust neural networks.
\newblock In {\em International Conference on Learning Representations}, 2019.

\bibitem{schmid2020multi}
Kyrill Schmid, Lenz Belzner, Thomy Phan, Thomas Gabor, and Claudia Linnhoff-Popien.
\newblock Multi-agent reinforcement learning for bargaining under risk and asymmetric information.
\newblock In {\em Proceedings of the 12th International Conference on Agents and Artificial Intelligence}, volume~1, pages 144--151, 2020.

\bibitem{stimpson2003learning}
Jeff~L Stimpson and Michael~A Goodrich.
\newblock Learning to cooperate in a social dilemma: A satisficing approach to bargaining.
\newblock In {\em International Conference on Machine Learning}, pages 728--735, 2003.

\bibitem{ueda2024fairness}
Daiju Ueda, Taichi Kakinuma, Shohei Fujita, Koji Kamagata, Yasutaka Fushimi, Rintaro Ito, Yusuke Matsui, Taiki Nozaki, Takeshi Nakaura, Noriyuki Fujima, et~al.
\newblock Fairness of artificial intelligence in healthcare: review and recommendations.
\newblock {\em Japanese Journal of Radiology}, 42(1):3--15, 2024.

\bibitem{van2021decaf}
Boris van Breugel, Trent Kyono, Jeroen Berrevoets, and Mihaela van~der Schaar.
\newblock Decaf: Generating fair synthetic data using causally-aware generative networks.
\newblock {\em Advances in Neural Information Processing Systems}, 34:22221--22233, 2021.

\bibitem{verma2018fairness}
Sahil Verma and Julia Rubin.
\newblock Fairness definitions explained.
\newblock In {\em Proceedings of the International Workshop on Software Fairness}, pages 1--7, 2018.

\bibitem{wachter2021fairness}
Sandra Wachter, Brent Mittelstadt, and Chris Russell.
\newblock Why fairness cannot be automated: Bridging the gap between eu non-discrimination law and ai.
\newblock {\em Computer Law \& Security Review}, 41:105567, 2021.

\bibitem{wan2023processing}
Mingyang Wan, Daochen Zha, Ninghao Liu, and Na~Zou.
\newblock In-processing modeling techniques for machine learning fairness: A survey.
\newblock {\em ACM Transactions on Knowledge Discovery from Data}, 17(3):1--27, 2023.

\bibitem{wei2022comprehensive}
Tianxin Wei and Jingrui He.
\newblock Comprehensive fair meta-learned recommender system.
\newblock In {\em Proceedings of the 28th ACM SIGKDD Conference on Knowledge Discovery and Data Mining}, pages 1989--1999, 2022.

\bibitem{yan2022forml}
Bobby Yan, Skyler Seto, and Nicholas Apostoloff.
\newblock Forml: Learning to reweight data for fairness.
\newblock {\em arXiv:2202.01719}, 2022.

\bibitem{yeh2009comparisons}
I-Cheng Yeh and Che-hui Lien.
\newblock The comparisons of data mining techniques for the predictive accuracy of probability of default of credit card clients.
\newblock {\em Expert Systems with Applications}, 36(2):2473--2480, 2009.

\bibitem{yu2020gradient}
Tianhe Yu, Saurabh Kumar, Abhishek Gupta, Sergey Levine, Karol Hausman, and Chelsea Finn.
\newblock Gradient surgery for multi-task learning.
\newblock {\em Advances in Neural Information Processing Systems}, 33:5824--5836, 2020.

\bibitem{zeng2023meta}
Yi~Zeng, Minzhou Pan, Himanshu Jahagirdar, Ming Jin, Lingjuan Lyu, and Ruoxi Jia.
\newblock $\{$Meta-Sift$\}$: How to sift out a clean subset in the presence of data poisoning?
\newblock In {\em 32nd USENIX Security Symposium (USENIX Security 23)}, pages 1667--1684, 2023.

\end{thebibliography}
\clearpage

\newpage
\appendix
\onecolumn
\section{Appendix}
In this section, we first provide background and related work ($\S$\ref{app:related_work}), along with context on the motivation of choosing NBS ($\S$\ref{app:discussion}). 
Then, we give formal definitions of some preliminaries mentioned in the main text ($\S$\ref{app:prelim}). We supplement discussions on the Bargaining game setup and assumptions ($\S$\ref{app:setup}), followed by the full proofs of corollaries and theorems ($\S$\ref{app:proof}). We then provide detailed experimental settings ($\S$\ref{app:exp_settings}) and the additional results and analysis on with 95\% CI, test noise, gradient conflict resolution, and experiments on the imbalanced image classification task ($\S$\ref{app:additional_result}). Finally, we discuss the broader impact and limitations ($\S$\ref{app:broader_impact}). 
We also supplement our code in \texttt{clean\_code.zip}, which will be made publicly available on GitHub after the anonymous period.

\subsection{Background \& Related Work}
\label{app:related_work}

\paragraph{Bargaining game in ML.} Bargaining has been broadly applied across various ML contexts such as multi-task learning \cite{navon2022multi}, multi-agent reinforcement learning \cite{stimpson2003learning, qiao2006multi, chang2020multi, schmid2020multi}, multi-armed bandits \cite{baek2021fair}, feature selection \cite{keshanian2022features}, and Bayesian optimization \cite{binois2020kalai}. 
For bargaining-related gradient aggregation, a key distinction between our work and prior research is our provision of a proof for the NBS that does not rely on the linear independence assumption, and we further extend from the basic gradient descent setting to bi-level fairness-aware meta-learning \cite{navon2022multi}. In the context of meta-learning for fairness, we critically examine and build upon the work of \cite{ren2018learning}, as detailed in $\S$\ref{sec:problem}. To the best of our knowledge, we are the first work to incorporate Nash bargaining into fairness-aware meta-learning.

\paragraph{Bias mitigation.} Traditional bias mitigation methods such as relabeling \cite{luong2011k}, resampling \cite{iosifidis2019fae}, or reweighting \cite{kamiran2012data} hinge upon the availability of sensitive attributes in training data. However, such attributes can often be inaccurate, incomplete, or entirely unavailable due to privacy and ethical concerns in the collection process \cite{ashurst2023fairness}. In response, researchers have developed two lines of proxy-based strategies over the past decades. The first line of work adopts indirect features associated with the sensitive feature of interest, such as zip codes for ethnicity \cite{datta2017proxy} or sound pitch for gender \cite{kum2022pseudo}. Despite circumventing the requirement of sensitive attributes in training data, the effectiveness of these proxies critically hinges on their correlation with the actual sensitive attributes. The second line of work aims at aligning with the Max-Min fairness principle and uses the worst-performing samples as the proxy of the most disadvantaged groups \cite{rawls2001justice, hashimoto2018fairness,lahoti2020fairness,chai2022self}. Challenges also arise in the potential bias toward mislabeled data when targeting worst-performing samples \cite{wan2023processing}. Most importantly, with the proxy group only aligned with Max-Min fairness, these lines of mitigation cannot be generalized or adapted to help other fairness notions.

\paragraph{Scope and Notions of Fairness.} Our exploration of fairness is specifically tailored to group-level notions, as fairness-aware meta-learning inherently relies on a representative set of samples organized by group information for model updates. While Max-Min fairness was chosen as a common evaluation metric, our paper also encompasses two additional mainstream fairness notions, including demographic parity \cite{yan2022forml} and average performance across equally represented groups \cite{ren2018learning}, both used as one-stage baselines (Section~\ref{sec:problem}). Our proposed two-stage method is also evaluated on its effectiveness in reducing disparity and improving the overall performance compared to these two baselines, respectively, as shown in Table~\ref{table:fairness}.

\paragraph{Max-Min Fairness and Rawlsian Justice.}
Group-level Max-Min fairness is a concept originated from Rawl's definition of fairness or, equivalently, justice \cite{rawls2001justice}. Rawls defines the least advantaged group by primary goods with objective characteristics, which are independent of specific predictors. The Difference Principle in Rawlsian Justice requires that the existing mechanism always contributes to the least advantaged group. While Rawlsian Justice has been extended to specific utilities in the context of ML group-level fairness \cite{hashimoto2018fairness,lahoti2020fairness,wan2023processing}, the worst-performing group in ML fairness may vary over epochs and depend on the optimization status, unlike the least advantaged group in the original Rawlsian context. Consequently, the group-level Max-Min fairness approach in ML may not necessarily create a safety net for the least advantaged group as the original Rawlsian Justice intends. Rather, the ML group-level Max-Min fairness provides a dynamic optimization strategy that maximizes the minimum performance across all groups at each epoch. This approach ensures that the worst-performing group, which may change over time, is prioritized during the optimization process. While this interpretation of fairness differs from the original Rawlsian context, it remains a valuable technique for promoting equity in ML systems by preventing any group from being consistently disadvantaged.

\subsection{Additional Discussions}
\label{app:discussion}

\textbf{Why NBS for hypergradient conflict resolution? -- A game theory perspective.} 
Game theory offers two main categories: cooperative and noncooperative games. Cooperative games involve players forming alliances to achieve common objectives, while noncooperative games (e.g., Nash equilibrium \cite{nash1950equilibrium}) focus on players acting independently without contracts. Our problem aligns better with the cooperative category, where players collaborate to maximize their collective gains from the proposed hypergradient update. Cooperative bargaining, a subset of cooperative games, studies how players with distinct interests negotiate to reach mutually beneficial agreements, directly corresponding to our goal of conflict resolution. Among cooperative bargaining solutions, the NBS stands out for its general applicability, robustness \cite{nash1953two}, and unique solution satisfying desirable axioms including Pareto Optimality, Symmetry, Invariance to Affine Transformations, and Independence of Irrelevant Alternatives \cite{nash1950bargaining}. These properties make NBS an attractive choice for resolving hypergradient conflicts in a principled and fair manner.

\begin{figure}[!h] 
    \centering
    \includegraphics[width=\linewidth]{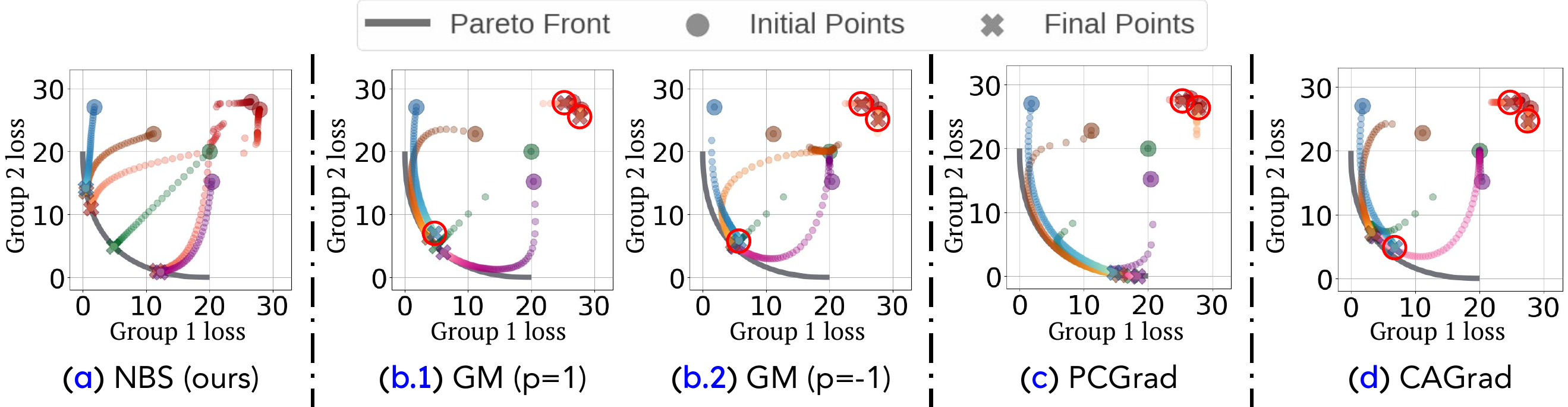}
    \caption{Synthetic illustration of each gradient aggregation method in resolving gradient conflicts and their implication in converging to Pareto front. Red circles imply the nodes that cannot converge to the Pareto front after 1000 steps of updates.
    }
    \label{fig:comparison}
\end{figure}

\textbf{Why NBS for hypergradient conflict resolution? -- An empirical comparison.} 
NBS has a unique ability to find a balanced solution on the Pareto front where no participant can unilaterally improve their position without others' agreement, which cannot be achieved by a standard way of linear scalarization (e.g., Example 2.27 in \cite{boyd2004convex}, and \cite{hu2024revisiting}) or other gradient aggregation methods like Generalized Mean (GM) \cite{bullen2013handbook}. Learning with NBS effectively steers the model toward the Pareto frontier during early training stages, which is the foundation of our proposed second stage (in our ``two-stage'' fairness-aware meta-learning) to achieve fairness goals without compromising model performance (the efficacy is illustrated in Figure \ref{fig:oursimu}). Empirically comparing optimization trajectories of NBS and other baselines of gradient aggregation using the same settings in Section \ref{sec:problem}: GM, PCGrad \cite{yu2020gradient}, and CAGrad \cite{liu2021conflict} (Figure \ref{fig:comparison}), we find that these methods, except NBS, often favor groups with larger loss magnitudes, resulting in inefficient convergence towards the Pareto frontier. In contrast, NBS, with its Pareto Optimality axiom, guides all nodes towards efficient convergence by considering the relative importance of different objectives and maximizing joint gains while avoiding favoring large loss value groups at the expense of others' utility.

\subsection{Preliminaries} 
\label{app:prelim}
We give a formal definition of Pareto Optimal and related terms, followed by the details of the four axioms of the NBS.

We write $x \succeq y$ as $x_i \geq y_i$ for all entry $i$ for vector $x, y$, and $x \succ y$ if $x \succeq y$ and $x \neq y$. We use $\|\cdot\|$ to denote $\ell^2$-norm.

\textbf{Pareto optimality}. Consider a set of function $\funcf_1, \ldots, \funcf_\groupcnt$ we want to minimize using some parameter $\param \in \paramspace$. Let vector valued function $\funcf(\param) = [\funcf_1(\param) \ldots \funcf_\groupcnt(\param)]$. $\param$ is Pareto optimal if for all other $\param' \in \paramspace$ satisfies $\funcf(\param') \succeq \funcf(\param)$. That is, no other $\param' \in \paramspace$ satisfies $\funcf(\param') \prec \funcf(\param)$, meaning that objectives cannot be jointly improved without sacrificing any of them.

\textbf{Pareto Front}. The set of Pareto Optimal models forms the Pareto front. There is no preference of the models in the Pareto front, unless with extra customized criteria (such as fairness) introduced.

\textbf{Pareto Improvement}. $\param'$ is an Pareto improvement to $\param$ if $\funcf(\param') \prec \funcf(\param)$.

\textbf{The four axioms of the NBS:}
\begin{enumerate}
    \item \textbf{Pareto Optimality:} The solution is Pareto Optimal in $\feasresult$. For solution $x \in \feasresult$, it does no exist $y \succ x$ for $y \in \feasresult, y \neq x$. 
    \item \textbf{Symmetry:} The solution is invariant to players' order.
    \item \textbf{Independence of Irrelevant Alternatives:} The solution retains upon superset expansion if it remains in the original set. That is, if we expand the feasible set $\feasresult$ and we know the solution stays in $\feasresult$, then the solution stays the same.
    \item \textbf{Invariant to Affine Transformations:} If we apply affine transformations to the utilities, the new solution is the original solution with transformed utilities accordingly. Consider affine transformations $\funcg_1, \ldots, \funcg_\groupcnt$. The original solution has utility payoffs $[x_1, \ldots, x_\groupcnt]$. If we transform the utilities from $\utility_i$ to $\funcg \circ \utility_i$, the new solution has payoffs $[\funcg_1(x_1), \ldots, \funcg_\groupcnt(x_\groupcnt)]$. 
\end{enumerate}

\subsection{Problem Setup \& Assumptions} \label{app:setup}

\paragraph{Additional assumptions.}
Assume that $\feasresult$ is convex and compact, and that there exists an $x \in \feasresult$ such that $x_i>\disresult_i$ for all players. We also assume that all players have complete information over the game parameterized as $(\feasresult, \disresult)$. Assume without loss of generality $ \groupgrad_i \neq 0$ for $i \in [\groupcnt]$.

\paragraph{Definition of $\alternative$.}

Recall that in Section~\ref{sec:method} we define $\alternative = \{ \nabla L:  \nabla L \in \ball, \nabla L_\alpha^\top \groupgrad_i - \dispoint^\top \groupgrad_i > 0,  \forall i \in [K] \}$ with $\dispoint = 0$. As a result, $\alternative \subset \ball$ and $\alternative$'s outer boundary lies on the boundary of $\ball$. For example, in 2D case, $\alternative$ is a set of circular sectors. A natural question is what happens when considering solutions in $\ball$. First note that we cannot directly take logarithm to obtain Equation~\ref{eq:utility_log_sum} as some $\nabla L_\alpha^\top \groupgrad_i - \dispoint^\top \groupgrad_i$ are non-positive. Yet, an alternative way is to adjust assumptions on $\dispoint$ such that $\nabla L_\alpha^\top \groupgrad_i - \dispoint^\top \groupgrad_i > 0, i \in [\groupcnt]$. For example, a way to guarantee the feasibility of $A$ is to set $\dispoint = \arg\min_{x \in B_\epsilon} x^\top g_i$ for $i \in [K]$ such that $(\nabla L_\alpha - d)^\top g_i > 0$ for all $i \in [K]$. Here, the feasible region can always be $B_\epsilon$. Assumptions on $\dispoint$ will be discussed later in the section. Under current assumption $\dispoint=0$, we can use the axiom of Independence of Irrelevant Alternatives to deduce a result on extending $A$ to $\ball$: if the new NBS stays in $A$, then it remains unchanged. This would serve as a theoretical guarantee if one wants to make assumptions on the solution under the constraint set $B_\epsilon$ instead. 
\paragraph{Feasibility of $\alternative$.}
Another assumption is that $\alternative$ is feasible. If $A$ is empty, then we cannot construct the Bargaining game because it doesn’t satisfy Nash's assumption that there always exists one payoff $x$ in the feasible payoff set such that $x$ is better than $d$ for each player. In this case, we use the default fairness protocol $\beta_0$ instead of the improved bargained outcome $\alpha$ as in Algorithm~\ref{alg:training}. Experiment results show that our mechanism can handle this situation adaptively. In fact, switching to the fairness protocol upon unsuccessful construction of bargaining offers a fresh starting point for subsequent bargaining iterations and helps to escape the saddle points (Figure~\ref{fig:alignment}, also discussed in Section~\ref{subsec:nash_meta_dynamics}).

\paragraph{Assumption $\dispoint = 0$.} We assume disagreement payoff $\dispoint = 0$. This is to address the hypergradient conflict such that the NBS satisfies $\nabla L_\alpha^\top g_i > 0, i \in [\groupcnt]$. We now discuss why we don’t manually adjust $\dispoint$ to make the bargaining problem feasible: Relaxing the constraint set will sacrifice the gradient alignment guarantee and deviate from our focus of conflict resolution. For $\dispoint \neq 0$, we are unsure about whether it would help resolve hypergradient conflict. This is because it may flip the sign of some terms, which might result in some $\nabla L_\alpha^\top g_i < 0$ and instead cause hypergradient conflict. 

\paragraph{General $\dispoint$.} Though general $\dispoint$ may not be suitable for conflict resolution (the problem identified in our paper), it may be useful for other scenarios as a gradient aggregation method (which is out of the scope of this work, but can be a future direction). Here, we include a discussion for completeness: If $A$ is feasible under general $\dispoint$, then we are able to get $\sum_i \log{(\nabla L_\alpha - \dispoint)^\top g_i}$ because each term is positive. The derivative w.r.t. $\nabla L_\alpha$ remains $\sum_i \frac{1}{(\nabla L_\alpha - \dispoint)^\top g_i} g_i$. Though closed-form solution may not be guaranteed because of the applicability of the tangent slope argument of Equation~\ref{eq:solve_for_upd} (that depends on the shape of $A$), but it can be solved as an optimization problem. Note that this may give a very different set coefficient because it is measured in terms of the improvement on $\dispoint$.  

\subsection{Proofs}
\label{app:proof}
\newtheorem*{theorem*}{Theorem}
\newtheorem*{corollary*}{Corollary}

\textbf{Proof of Theorem} \ref{thm:main_derivative}:

\begin{theorem*}
     Under $\dispoint=0$, $\arg\max_{\gradupd \in \alternative} \prod_{i\in[\groupcnt]} (u_i({\gradupd}) - d_i)$ is achieved at
    \begin{align}
    \sum_{i\in[\groupcnt]} \frac{1}{{\gradupd}^\top \groupgrad_i} g_i &= \const \gradupd, \quad \text{for some } \const > 0. 
    \end{align}
\end{theorem*}

\begin{proof}
     We follow the same steps in Claim 3.1 of \cite{navon2022multi}, along with additional  characterization on the shape of $\alternative$ under $\dispoint=0$. Note that this theorem may not work under general $\dispoint$.
     
     By the positivity of each term in $\alternative$, this is equivalent to maximizing the summation of logarithms:
    \begin{align} \label{eq:utility_log_sum}
    \arg\max_{\gradupd \in \ball} \sum_{i\in[\groupcnt]} \log({\gradupd}^\top \groupgrad_i) 
    \end{align} 

    Taking the derivative of the objective w.r.t. $\gradupd$ gives $ \sum_{i\in \groupcnt} \frac{1}{{\gradupd}^\top \groupgrad_i} \groupgrad_i$. For any $i \in [\groupcnt]$, we know $x^\top \groupgrad_i > 0$ if and only if $(c x)^\top \groupgrad >0$ with any given $c>0$. This means that if a point is in $\alternative$, then all points in its radial direction in $\ball$ are in $\alternative$ (i.e. the boundary of $\alternative$ is a subset of the boundary of $\ball$). By the Pareto Optimality, the optimal solution must lie on the boundary of $\ball$ as the utility is monotonically increasing in $\lVert \gradupd \rVert$ for $ \gradupd \in \alternative$.
    The optimal points on the boundary of $\ball$ have tangent slope 0 (i.e. gradient having the same direction as its normal). Hence, we know the normal is in parallel to $\gradupd$ and the desired $\gradupd$ satisfies
    \begin{align}
    \sum_{i\in[\groupcnt]} \frac{1}{{\gradupd}^\top \groupgrad_i} \groupgrad_i &= \const \gradupd, \quad \text{for some } \const > 0.
    \end{align}
\end{proof}

\textbf{Proof of Theorem} \ref{thm:solving_grad}:
\begin{theorem*} 
The solution to Equation \ref{eq:solve_for_upd} is (up to scaling) $\gradupd = \sum_{i\in \groupcnt} \gcoefunknown_i \groupgrad_i$ where 
\begin{align}
 \allgrad^\top \allgrad \gcoefunknown = \frac{1}{\gcoefunknown} 
\end{align}
with the element-wise reciprocal $\frac{1}{\gcoefunknown}$. 
\end{theorem*}
\begin{proof}
Multiplying both sides with $\groupgrad_j$, this is equivalent to solving for $\unknown$ in
$ \sum_{i \in [\groupcnt]} \frac{\groupgrad_i^\top \groupgrad_j}{\unknown^\top \groupgrad_i}  = \const \unknown^\top \groupgrad_j$ for $j \in [\groupcnt]$. 
 Observe that $\unknown$ is a linear combination of $\groupgrad_i$ (i.e. $\unknown = \frac{1}{\const}\sum_{i\in[\groupcnt]} (\unknown^\top \groupgrad_i)^{-1} \groupgrad_i$). It suffices to solve for coefficients $\unknown^\top \groupgrad_i$ by the linear system
\begin{align}
    \sum_{i\in\groupcnt} (\groupgrad_i^\top \groupgrad_j) (\unknown^\top \groupgrad_i)^{-1}=\const \unknown^\top \groupgrad_j \label{eq:coef_linear_system}
\end{align}
for $j  \in [\groupcnt]$. 
Without loss of generality, set $\const = 1$ to ascertain the direction of $\unknown$. Let $\gcoefunknown = [\gcoefunknown_1 \ldots \gcoefunknown_\groupcnt ]$ with $\gcoefunknown_i = (\unknown^\top \groupgrad_i)^{-1}$. 
Equation \ref{eq:coef_linear_system} becomes
\begin{align}
     \groupgrad_j^\top \sum_{i\in [\groupcnt]} \groupgrad_i \gcoefunknown_i = \gcoefunknown_j^{-1} 
\end{align}
for $j  \in [\groupcnt]$, or, equivalently, 
\begin{align}
    \allgrad^\top \allgrad  \gcoefunknown = \frac{1}{\gcoefunknown} 
\end{align}
with the element-wise reciprocal $\frac{1}{\gcoefunknown}$, concluding the proof. Note that $- \gcoefunknown$ is also a solution when $\gcoefunknown$ is one, yet we preserve $\gcoefunknown \in \mathbbm{R}^{\groupcnt}_{+}$ (i.e. positive contribution of each $\groupgrad_i$) in implementation.
\end{proof}

\textbf{Additional Note on Linear Independence Condition for Theorem }\ref{thm:solving_grad}:

In Section~\ref{subsec:nash_sol}, we highlight that linear independence does not hold in general. We also wanted to note that fine-grained subgroup definitions can make linear independence assumptions more realistic as they break down larger groups into smaller, distinct units. If subgroups are well-defined and distinct in their goals, characteristics, or attributes, the vectors representing their goals are more likely to be linearly independent. However, real-world scenarios may still exhibit interdependencies or correlations between subgroup goals, especially in the context of fairness, even if subgroups are finely defined. For example, goals may still rely on common resources or have shared objectives at a higher level. The interdependencies and shared factors may cause dependency in hypergradients. Specifically, if one goal can be expressed as a combination of the goals of other groups, then the corresponding vectors are linearly dependent. Additionally, if the number of subgroups exceeds the dimension of the hypergradient (i.e. definition too fine-grained), it’s impossible to have linear independence. Therefore, while fine-grained subgroup definitions can make linear independence assumptions more realistic, it's essential to carefully analyze the specific context and relationships between subgroups to determine the extent to which linear independence holds. 

\textbf{Proof of Corollary} \ref{corollary:l2norm}:
\begin{corollary*}
(Norm of bargained update) The solution in Theorem \ref{thm:solving_grad} has $\ell^2$-norm $\sqrt{\groupcnt}$.
\end{corollary*}
\begin{proof}
It follows that 
\begin{align}
    \lVert \sum_{i\in \groupcnt} \gcoefunknown_i \groupgrad_i \rVert^2= \sum_{i\in \groupcnt} \sum_{j\in \groupcnt} \gcoefunknown_i \gcoefunknown_j \groupgrad_i^\top \groupgrad_j = \groupcnt.
\end{align}
\end{proof}

\textbf{Proof of Corollary} \ref{corollary:alpha_bdd}:
\begin{corollary*}
If $\groupgrad_j$ is $\bdd$-bounded for $j \in [\groupcnt]$, $ \lVert \gcoefunknown_j^{-1} \rVert$ is $(\sqrt{\groupcnt}\bdd)$-bounded for $j\in [\groupcnt]$. 
\end{corollary*}
\begin{proof}
By Cauchy-Schwarz inequality, we have 
\begin{align}
    \lVert \gcoefunknown_j^{-1} \rVert &= \left\lVert \groupgrad_j^\top \sum_{i\in[\groupcnt]} \gcoefunknown_i \groupgrad_i \right\rVert \leq \left\lVert \groupgrad_j^\top \right\rVert \left\lVert \sum_{i\in[\groupcnt]} \gcoefunknown_i \groupgrad_i \right\rVert \leq \sqrt{\groupcnt}\bdd.
\end{align}
\end{proof}

\textbf{Proof of Theorem} \ref{thm:theta_upd_rule}:
\begin{theorem*}
    (Update rule of $\param$) Denote $\trainlosstmp_i = \loss(\minibatch_i^{(t)}|\param^{(t)}) \in \mathbb{R}$ for the $i$-th sample in training minibatch $\minibatch^{(t)}$ at step $t$. $\param$ is updated as 
 $
        \param^{(t+1)} = \param^{(t)} -  \frac{\lr^{(t)}}{|\minibatch^{(t)}|} \sum_{i =1}^{|\minibatch^{(t)}|} \Delta \param^{i}
 $
 with $\Delta \param^{i} = \max{((\nabla_\param ((\fairnessaggweight^{(t)})^\top \vallosstmp))^\top \nabla_\param \trainlosstmp_i, 0)} \nabla_\param \trainlosstmp_i$.
\end{theorem*}
\begin{proof}
    Since we use the meta-learning framework in \cite{ren2018learning} with customized $\fairnessaggweight^{(t)}$, we first evaluate the effect of ${\fairnessaggweight^{(t)}}$ in the computation graph, following similar steps in Equation 12 and Appendix A of \cite{ren2018learning}. We initialize $\intermweight$ to $0$. 
    For data sample $i$ in training batch $\minibatch^{(t)}$ at step $t$, we perform a single gradient update as in Algorithm \ref{alg:training}:
    \begin{align}
        \intermweight_{i}^{(t)} &=  \left. \frac{\partial }{\partial \intermweight_{i}^{(t)}} (\fairnessaggweight^{(t)})^\top \groupvalloss(\valdata|\hat{\param}^{(t)}) \right\vert_{\intermweight_i^{(t)}=0} \\
        &=  \left. \frac{\partial}{\partial \param} (\fairnessaggweight^{(t)})^\top  \groupvalloss(\valdata | \param) \right\vert_{\param=\param^{(t)}}^\top \left. \frac{\partial }{\partial \reweight_{i}^{(t)}} \hat{\param}^{(t)} (\intermweight_{i}^{(t)}) \right\vert_{\intermweight_i^{(t)}=0} \label{eq:comp_graph_1} \\
        & \propto -   \left. \frac{\partial}{\partial \param} (\fairnessaggweight^{(t)})^\top \groupvalloss(\valdata | \param) \right\vert_{\param=\param^{(t)}}^\top \left. \frac{\partial}{\partial \param} \loss(\minibatch^{(t)}_i|\param)  \right\vert_{\param=\param^{(t)}} \label{eq:comp_graph_2}.
    \end{align}
    Line \ref{eq:comp_graph_1} comes from chain rule; Line \ref{eq:comp_graph_2} comes from 
    \begin{align}
       \frac{\partial}{\partial \intermweight_i^{(t)}} \hat{\param}^{(t)}(\intermweight_i^{(t)}) &= \frac{\partial}{\partial \intermweight_i^{(t)}} \left( \param^{(t)} - \lr^{(t)} \nabla_{\param^{(t)}} \intermweight^{(t)} \cdot \loss(\minibatch^{(t)} | \param^{(t)}) \right)\\
       &= - \lr^{(t)} \left. \frac{\partial}{\partial \param} \loss(\minibatch^{(t)}_i|\param)  \right\vert_{\param=\param^{(t)}} 
    \end{align}
    Informally, it means that $\intermweight$, or $\reweight$, equivalently, is jointly determined by the ${\fairnessaggweight^{(t)}}$-weighted gradient on the meta set and the gradient on the training batch.
    Analogous to Equation 30 of \cite{ren2018learning}, the update rule of $\param$ is
    \begin{align}
        \param^{(t+1)} &= \param^{(t)} - \lr^{(t)} \frac{\partial}{\partial \theta}\frac{1}{|\minibatch^{(t)}|} \sum_{i =1}^{|\minibatch^{(t)}|} \max{(-\intermweight_i^{(t)}, 0)} \loss(\minibatch_i^{(t)}| \param^{(t)}) \\
        &= \param^{(t)} -  \frac{\lr^{(t)}}{|\minibatch^{(t)}|} \sum_{i =1}^{|\minibatch^{(t)}|} \Delta \param^{i} \\
        \text{with } \Delta \param^{i} &= \max{((\nabla_\param ((\fairnessaggweight^{(t)})^\top \vallosstmp))^\top \nabla_\param \trainlosstmp_i, 0)} \nabla_\param \trainlosstmp_i.
    \end{align}
    Note that losses and gradients are taken w.r.t. $\param^{(t)}$, and
    without loss of generality we can disregard $\lr^{(t)}$ as if it absorbs the normalization constant. 
\end{proof}
\textbf{Proof of Theorem} \ref{thm:pareto_improvement}:
\begin{theorem*}
       (Pareto improvement of $\intermweight$) Use $\alpha^{(t)}$ for the update. Assume $\groupvalloss_i$ is Lipschitz-smooth with constant $\lsmooth$ and $\groupgrad_i^{(t)}$ is $\bdd$-bounded at step $t$. If the meta learning rate for $\intermweight$ satisfies $\metalr^{(t)} \leq \frac{2}{\lsmooth \groupcnt \gcoefunknown_j^{(t)}}$ for $j\in [\groupcnt]$, then $\groupvalloss_i(\intermweight^{(t+1)})  \leq \groupvalloss_i(\intermweight^{(t)})$ for any group $i \in [\groupcnt]$.
\end{theorem*}
\begin{proof}
    We follow similar (and standard) steps in Theorem 5.4 of \cite{navon2022multi}, yet retrieve a slightly different (and group-wise tighter) upperbound for learning rate. In Theorem 5.4 of \cite{navon2022multi}, the upperbound for learning rate is $\min_{i \in [\groupcnt]} \frac{1}{\lsmooth \groupcnt \gcoefunknown_j^{(t)}}$. Our bound is better with multiplicative constant 2 than that of \cite{navon2022multi}.
    
    Write $\Delta \intermweight = \intermweight^{(t+1)} - \intermweight^{(t)}$ and use a well-known property of Lipschitz-smoothness:
    \begin{align}
        \groupvalloss_i(\intermweight^{(t+1)}) &\leq \groupvalloss_i(\intermweight^{(t)}) - I_1 + I_2 \\
        \text{with } I_1 &= \metalr^{(t)} \nabla \groupvalloss_i(\intermweight^{(t)})^\top \Delta \intermweight\\
        &= \metalr^{(t)} \groupgrad_i^{(t)\top} \sum_{j=1}^{\groupcnt} \gcoefunknown^{(t)}_j \groupgrad^{(t)}_j  \\
        &= \frac{\metalr^{(t)}}{\gcoefunknown^{(t)}_i},\\
        I_2 &= \frac{\lsmooth}{2} \lVert \metalr^{(t)} \Delta \intermweight \rVert^2.
    \end{align}
    By Corollary \ref{corollary:l2norm}, 
    \begin{align}
        I_2 &= \frac{\lsmooth}{2} \lVert \metalr^{(t)} \Delta \intermweight \rVert^2 = \frac{C (\metalr^{(t)})^2 }{2} \lVert \Delta \intermweight \rVert^2 = \frac{C K (\metalr^{(t)})^2 }{2}.
    \end{align}
    Observe that $-I_1 + I_2 \leq 0$ is equivalent to 
    \begin{align}
        \frac{\lsmooth \groupcnt (\metalr^{(t)})^2}{2} &\leq \frac{\metalr^{(t)}}{\gcoefunknown^{(t)}_i} \\
        \metalr^{(t)} &\leq \frac{2}{\lsmooth \groupcnt \gcoefunknown^{(t)}_i}  \label{eq:meta_lr_bd}.
    \end{align}
    We also have $\frac{2}{\lsmooth \groupcnt \gcoefunknown^{(t)}_i} \leq  \frac{2 \sqrt{\groupcnt}\bdd}{\lsmooth \groupcnt} = \frac{2 \bdd}{\lsmooth \sqrt{\groupcnt}}$  by Corollary \ref{corollary:alpha_bdd}. Note that $\groupvalloss_i(\intermweight^{(t+1)})  \leq \groupvalloss_i(\intermweight^{(t)})$ is strict when Inequality~\ref{eq:meta_lr_bd} is strict, which yields Pareto Improvement. We just perform one step of gradient update of $\intermweight$ in Algorithm \ref{alg:training}, yet this result is applicable to single level optimization that directly optimize the parameter in interest.  
\end{proof}

\textbf{Proof of Theorem \ref{thm:monotone_val}}:
We will use the following corollary in the proof:
\begin{corollary}
    Assume $\funcf: \mathbb{R}^\paramdim \rightarrow \mathbb{R}^\groupcnt$ is Lipschitz-smooth with constant $\lsmooth$. Fix $\fairnessaggweight \in \mathbb{R}^\groupcnt$ with finite $\lVert \fairnessaggweight \rVert$. For $\funcg: \mathbb{R}^\groupcnt \rightarrow \mathbb{R}$, $\funcg(x) = \fairnessaggweight^\top \funcf(x)$. Then $\funcg$ is Lipschitz-smooth with constant $\lsmooth \lVert \fairnessaggweight \rVert$. \label{corollary:L_beta_smooth}
\end{corollary}
\begin{proof}
    This is a classical result in Real Analysis. By chain rule, we have $\nabla \funcg(x) = (\nabla \funcf(x))\fairnessaggweight$. Then for any $x, y \in \mathbb{R}^\paramdim$, by Cauchy-Schwarz inequality,
    \begin{align}
        \lVert \nabla \funcg(x) - \nabla \funcg(y) \rVert = \lVert (\nabla \funcf(x))\fairnessaggweight - (\nabla \funcf(y))\fairnessaggweight \rVert =  \lVert (\nabla \funcf(x)-  \nabla \funcf(y))\fairnessaggweight \rVert \leq   \lVert \nabla \funcf(x)-  \nabla \funcf(y)  \rVert \lVert \fairnessaggweight  \rVert \leq  \lsmooth \lVert \fairnessaggweight \rVert \lVert x-y \rVert.
    \end{align}
    Thus $\funcg$ is Lipschitz-smooth with constant $\lsmooth \lVert \fairnessaggweight \rVert$.
\end{proof}

\begin{theorem*}
    (Monotonic improvement of validation loss w.r.t. $\param$)
    Assume $\vallosstmp$ is Lipschitz-smooth with constant $\lsmooth$ and $\nabla_{\param} \trainlosstmp_i$ is $\bdd$-bounded. If the learning rate for $\param$ satisfies $\lr^{(t)} \leq \frac{2 |\minibatch^{(t)}| }{ \lsmooth \lVert \fairnessaggweight^{(t)} \rVert \bdd^2}$, then $ \loss_{\fairnessaggweight^{(t)}} (\param^{(t + 1)}) \leq \loss_{\fairnessaggweight^{(t)}} (\param^{(t)})$ for any fixed vector $\fairnessaggweight^{(t)}$ with finite $\lVert \fairnessaggweight^{(t)} \rVert$ used to update $\param^{(t)}$.
\end{theorem*}
\begin{proof}
    We incorporate $\fairnessaggweight^{(t)}$ in the computation when following the same (and standard) steps in Lemma 1 of \cite{ren2018learning}. We drop the superscript of $\fairnessaggweight^{(t)}$ for simplicity. By Corollary \ref{corollary:L_beta_smooth}, we know $\fairnessloss$ is Lipschitz-smooth with constant  $\lsmooth \lVert \fairnessaggweight \rVert$. Similar to Theorem \ref{thm:pareto_improvement}, we use the gradients derived in Theorem \ref{thm:theta_upd_rule} and have
    \begin{align}
       \fairnessloss(\param^{(t+1)}) &\leq  \fairnessloss(\param^{(t)}) - I_1 + I_2 \\
      \text{with } I_1 &= (\nabla \fairnessloss)^\top \Delta \param \\
        &= (\nabla \fairnessloss)^\top \frac{\lr^{(t)}}{|\minibatch^{(t)}|} \sum_{i \in \minibatch^{(t)}} \max \{ (\nabla \fairnessloss)^\top \nabla \trainlosstmp_i, 0\} \nabla \trainlosstmp_i\\
        &= \frac{\lr^{(t)}}{|\minibatch^{(t)}|} \sum_{i \in \minibatch^{(t)}} \max \{ ((\nabla \fairnessloss)^\top \nabla \trainlosstmp_i)^2, 0\}\\
        I_2 &= \frac{\lsmooth \lVert \fairnessaggweight \rVert}{2} \lVert \Delta \param \rVert^2 \\
        &= \frac{\lsmooth \lVert \fairnessaggweight \rVert}{2} \left\lVert \frac{\lr^{(t)}}{|\minibatch^{(t)}|} \sum_{i \in \minibatch^{(t)}} \max \{ (\nabla \fairnessloss)^\top \nabla \trainlosstmp_i, 0 \} \nabla \trainlosstmp_i  \right\lVert^2 \\
        & \leq \frac{\lsmooth \lVert \fairnessaggweight \rVert}{2}  \frac{ (\lr^{(t)})^2 }{|\minibatch^{(t)}|^2} \sum_{i \in \minibatch^{(t)}} \left\lVert \max \{ (\nabla \fairnessloss)^\top \nabla \trainlosstmp_i, 0 \} \nabla \trainlosstmp_i  \right\rVert^2 \label{eq:i2_ineq1}\\
        & \leq \frac{\lsmooth \lVert \fairnessaggweight \rVert}{2}  \frac{ (\lr^{(t)})^2 }{|\minibatch^{(t)}|^2} \sum_{i \in \minibatch^{(t)}} \max \{ ((\nabla \fairnessloss)^\top \nabla \trainlosstmp_i)^2, 0 \} \left\lVert \nabla \trainlosstmp_i  \right\rVert^2 \label{eq:i2_ineq2}\\
        & \leq \frac{\lsmooth \lVert \fairnessaggweight \rVert (\lr^{(t)})^2  \bdd^2 }{2|\minibatch^{(t)}|^2}  \sum_{i \in \minibatch^{(t)}} \max \{ ((\nabla \fairnessloss)^\top \nabla \trainlosstmp_i)^2, 0 \} 
    \end{align}
     Line \ref{eq:i2_ineq1} and  \ref{eq:i2_ineq2} come from the triangle inequality and Cauchy-Schwarz inequality. Note that here the gradients are taken w.r.t. $\param$. We know $-I_1 + I_2 \leq 0$ is equivalent to  
     \begin{align}
         \lr^{(t)} \leq \frac{2 |\minibatch^{(t)}|}{\lsmooth \lVert \fairnessaggweight \rVert \bdd^2},
     \end{align}
     concluding the proof.
\end{proof}


\subsection{Detailed Experimental Settings}
\label{app:exp_settings}

\paragraph{Synthetic settings (used in $\S$\ref{sec:problem}, $\S$\ref{sec:evaluation}).}
We provide here the details for the illustrative example of Figure \ref{fig:problem} and \ref{fig:oursimu}. We use a modified version of the illustrative example in \cite{navon2022multi}. We first present our modified learning problem: Let $\theta = (\theta_1, \theta_2) \in \mathbb{R}^2$, and consider the following objectives:
\begin{align*}
    & \ell_1(\theta) = c_1(\theta)f_1(\theta) + c_2(\theta)g_1(\theta) \text{  and  }  
    \ell_2(\theta) = c_1(\theta)f_2(\theta) + c_2(\theta)g_2(\theta),  \text{  where  }  \\
    & f_1(\theta) = \log(\max(|0.5(-\theta_1 - 7) - \tanh(-\theta_2)|, 5e - 6)) + 6, \\
    & f_2(\theta) = \log(\max(|0.5(-\theta_1 + 3) - \tanh(-\theta_2) + 2|, 5e - 6)) + 6, \\
    & g_1(\theta) = ((-\theta_1 + 7)^2 + 0.1 \cdot (-\theta_2 - 8)^2)/10 - 20, \\
    & g_2(\theta) = ((-\theta_1 - 7)^2 + 0.1 \cdot (-\theta_2 - 8)^2)/10 - 20, \\
    & c_1(\theta) = \max(\tanh(0.5\theta_2), 0) \text{  and  } c_2(\theta) = \max(\tanh(-0.5\theta_2), 0)
\end{align*}
We now set $L_1 = \ell_1$ and $L_2 = \ell_2$ as our objectives to simulate the case where two participated groups have the same scale of loss, as a closer reflection of the loss scope may encountered during inter-group bargaining. In particular, with the loss scale being set to the same, the fairness goal of the evaluated notions, i.e., LtR, FORML, and Meta-gDRO, becomes the same as \textbf{x=y}. We use six different initialization points $\{(-8.5, 7.5), (0.0, 0.0), (9.0, 9.0), (-7.5, -0.5), (9, -1.0), (9, -20)\}$. We use the SGD optimizer and train each method for 1000 iterations with a learning rate of $0.1$.

\paragraph{Standard fairness benchmark settings (used in $\S$\ref{sec:evaluation}).}
The detailed fairness dataset settings are provided in Table \ref{tab:data_summary}.
We access these fairness datasets via Library for Semi-Automated Data Science: \url{https://github.com/IBM/lale}. The library is under \href{https://github.com/IBM/lale/blob/master/LICENSE.txt}{Apache 2.0} License.
\begin{table}[h!]
\centering
\resizebox{\linewidth}{!}{
\begin{tabular}{l|ccccccc}
\toprule
\multicolumn{1}{l|}{\textbf{Dataset}}                                                                           & \textbf{\# samples}         & \textbf{\# features}        & \textbf{Favor label}     & \textbf{Attribute}                 & \textbf{Trainning Set}                                                                                                                                                                                      & \textbf{Test set (3\% of whole data)}                                                                                                                                                                             & \textbf{Validation set}                                                                                                                                                                      \\ \midrule \midrule
                                                                                                &                             &                             &                          & sex                                & \begin{tabular}[c]{@{}c@{}}M (\{0: 22365, 1: 9555\})\\ F (\{0: 14060, 1: 1430\})\end{tabular}                                                                                                               & \begin{tabular}[c]{@{}c@{}}M (\{0: 366, 1: 366\})\\ F (\{0: 366, 1: 366\})\end{tabular}                                                                                                                          & \begin{tabular}[c]{@{}c@{}}M (\{0: 3, 1: 3\})\\ F (\{0: 3, 1: 3\})\end{tabular}                                                                                                              \\ \cmidrule(l){5-8}
\multirow{-2}{*}{\textbf{\begin{tabular}[c]{@{}l@{}}Adult \\ Income\end{tabular}}}              & \multirow{-2}{*}{48842}     & \multirow{-2}{*}{105}       & \multirow{-2}{*}{1}      & race                               & \begin{tabular}[c]{@{}c@{}}Black (\{0: 3972, 1: 428\})\\ White (\{0: 31006, 1: 10458\})\\ Asian (\{0: 969, 1: 282\})\\ Other (\{0: 233\})\\ Amer-Indian (\{0: 289, \scalebox{1}{\colorbox[HTML]{FFEEC2}{1: 1}}\})\end{tabular} & \begin{tabular}[c]{@{}c@{}}Black (\{0: 146, 1: 146\})\\ White (\{0: 146, 1: 146\})\\ Asian (\{0: 146, 1: 146\})\\ Other (\{0: 146, 1: 146\})\\ Amer-Indian (\{0: 146, 1: 146\})\end{tabular} & \begin{tabular}[c]{@{}c@{}}Black (\{0: 3, 1: 3\})\\ White (\{0: 3, 1: 3\})\\ Asian (\{0: 3, 1: 3\})\\ Other (\{0: 3, 1: 3\})\\ Amer-Indian (\{0: 3, 1: 3\})\end{tabular} \\ \midrule
\rowcolor[HTML]{E0E0E0} 
\textbf{\begin{tabular}[c]{@{}l@{}}Bank\\ Telemarketing\end{tabular}}                           & 45211                       & 51                          & 1                        & age                                & \begin{tabular}[c]{@{}c@{}}age $>$ 25 (\{1: 38564, 0: 4632\})\\ age $<$ 25 (\{1: 715, 0: 108\})\end{tabular}                                                                                                & \begin{tabular}[c]{@{}c@{}}age $>$ 25 (\{1: 339, 0: 339\})\\ age $<$ 25 (\{1: 339, 0: 339\})\end{tabular}                                                                                                        & \begin{tabular}[c]{@{}c@{}}age $>$ 25 (\{1: 3, 0: 3\})\\ age $<$ 25 (\{1: 3, 0: 3\})\end{tabular}                                                                                            \\ \midrule
\textbf{\begin{tabular}[c]{@{}l@{}}Credit \\ Default\end{tabular}}                              & 30000                       & 24                          & 0                        & sex                                & \begin{tabular}[c]{@{}c@{}}M (\{1: 14122, 0: 3542\})\\ F (\{1: 8787, 0: 2656\})\end{tabular}                                                                                                                & \begin{tabular}[c]{@{}c@{}}M (\{1: 225, 0: 225\})\\ F (\{1: 225, 0: 225\})\end{tabular}                                                                                                                          & \begin{tabular}[c]{@{}c@{}}M (\{1: 4, 0: 4\})\\ F (\{1: 4, 0: 4\})\end{tabular}                                                                                                              \\ \midrule
\rowcolor[HTML]{E0E0E0} 
{\color[HTML]{000000} \textbf{\begin{tabular}[c]{@{}l@{}}Communities\\ and Crime\end{tabular}}} & {\color[HTML]{000000} 1994} & {\color[HTML]{000000} 1929} & {\color[HTML]{000000} 0} & {\color[HTML]{000000} blackgt6pct} & {\color[HTML]{000000} \begin{tabular}[c]{@{}c@{}}False (\{1: 1002, \scalebox{1}{\colorbox[HTML]{FFEEC2}{0: 0}}\})\\ True (\{1: 841, 0: 101\})\end{tabular}}                                                                                        & {\color[HTML]{000000} \begin{tabular}[c]{@{}c@{}}False (\{1: 14, 0: 14\})\\ True (\{1: 14, 0: 14\})\end{tabular}}                                                                                                & {\color[HTML]{000000} \begin{tabular}[c]{@{}c@{}}False (\{1: 1, 0: 1\})\\ True (\{1: 1, 0: 1\})\end{tabular}}                                                                                \\ \midrule
\textbf{\begin{tabular}[c]{@{}l@{}}Titanic\\ Survival\end{tabular}}                             & 1309                        & 1526                        & 1                        & sex                                & \begin{tabular}[c]{@{}c@{}}M (\{1: 151, 0: 672\})\\ F (\{1: 329, 0: 118\})\end{tabular}                                                                                                                     & \begin{tabular}[c]{@{}c@{}}M (\{1: 9, 0: 9\})\\ F (\{1: 9, 0: 9\})\end{tabular}                                                                                                                                  & \begin{tabular}[c]{@{}c@{}}M (\{1: 1, 0: 1\})\\ F (\{1: 1, 0: 1\})\end{tabular}                                                                                                              \\ \midrule
\rowcolor[HTML]{E0E0E0} 
\textbf{\begin{tabular}[c]{@{}l@{}}Student\\ Performance\end{tabular}}                          & 649                         & 58                          & 1                        & sex                                & \begin{tabular}[c]{@{}c@{}}M (\{1: 200, 0: 33\})\\ F (\{1: 315, 0: 32\})\end{tabular}                                                                                                                       & \begin{tabular}[c]{@{}c@{}}M (\{1: 16, 0: 16\})\\ F (\{1: 16, 0: 16\})\end{tabular}                                                                                                                              & \begin{tabular}[c]{@{}c@{}}M (\{1: 2, 0: 2\})\\ F (\{1: 2, 0: 2\})\end{tabular}                                                                                                              \\

\bottomrule
\end{tabular}
}
\vspace{0.7em}
\caption{We detail the number of samples per group, categorized by protected attributes and labels. The student performance dataset allocates 10\% for testing due to its smaller size, while others reserve 3\%. Notably, in the adult income training dataset with race as a protected attribute, only one Amer-Indian sample with the positive (favorable) label. Additionally, after balancing test and validation sets, the community and crime dataset has no samples in the ``False'' group labeled ``0''. These training data distribution-wise problems are marked in \scalebox{0.85}{\colorbox[HTML]{FFEEC2}{yellow}} and were maintained to examine how extreme imbalances in training data impact various algorithms.}
\label{tab:data_summary}
\end{table}

\paragraph{Training specifics for standard fairness benchmarks (used in $\S$\ref{sec:evaluation}).}
For the models applied to these fairness datasets, we consistently employ a 3-layer neural network architecture comprised of an input layer, one hidden layer, and an output layer. The input layer takes in features and transforms them to a dimension of size 128. A ReLU activation function and a dropout layer for regularization follow this. The hidden layer further processes the data, again followed by a ReLU and dropout layer. Finally, the output layer maps the representation from the hidden layer to the number of classes specified (2 for the considered fairness datasets). This architecture utilizes dropout after each ReLU activation to reduce overfitting, making it suitable for classification tasks.

We fine-tuned the training hyperparameters based on the baseline model. Common hyperparameters across all algorithms include a total of 50 training epochs, an SGD optimizer momentum of 0.9, and a weight decay of 5e-4, with the bargaining phase limited to 15 epochs for the three settings incorporating proposed Nash-Meta-Learning. Hyperparameters that varied are detailed in Table \ref{tab:hyperparameters}.
\paragraph{Computing.} All experiments were conducted on an internal cluster using one chip of H-100.

\begin{table}[h!]
\centering
\resizebox{0.56\linewidth}{!}{
\begin{tabular}{@{}lcccc@{}}
\toprule
\multicolumn{1}{l|}{\textbf{Dataset}}                                                                                                                         & \textbf{Attribute}                  & \textbf{Learning rate}      & \textbf{Dropout Probability} & \textbf{Batch Size}       \\ \midrule  \midrule
\multicolumn{1}{l|}{}                                                                                                                        & sex                                 & 1e-3                        & 0.2                          & 512                       \\ \cmidrule(l){2-5} 
\multicolumn{1}{l|}{\multirow{-2}{*}{\textbf{\begin{tabular}[c]{@{}l@{}}Adult \\ Income\end{tabular}}}}                                      & race                                & 5e-4                        & 0.4                          & 512                       \\ \midrule
\rowcolor[HTML]{E0E0E0} 
\multicolumn{1}{l|}{\cellcolor[HTML]{E0E0E0}\textbf{\begin{tabular}[c]{@{}l@{}}Bank\\ Telemarketing\end{tabular}}}                           & age                                 & 1e-3                        & 0.3                          & 512                       \\ \midrule
\multicolumn{1}{l|}{\textbf{\begin{tabular}[c]{@{}l@{}}Credit \\ Default\end{tabular}}}                                                      & sex                                 & 1e-3                        & 0.5                          & 512                       \\ \midrule
\rowcolor[HTML]{E0E0E0} 
\multicolumn{1}{l|}{\cellcolor[HTML]{E0E0E0}{\color[HTML]{000000} \textbf{\begin{tabular}[c]{@{}l@{}}Communities\\ and Crime\end{tabular}}}} & {\color[HTML]{000000} blackgt6pact} & {\color[HTML]{000000} 1e-4} & {\color[HTML]{000000} 0.2}   & {\color[HTML]{000000} 32} \\ \midrule
\multicolumn{1}{l|}{\textbf{\begin{tabular}[c]{@{}l@{}}Titanic\\ Survival\end{tabular}}}                                                     & sex                                 & 1e-3                        & 0.4                          & 32                        \\ \cmidrule(l){2-5} 
\rowcolor[HTML]{E0E0E0} 
\multicolumn{1}{l|}{\cellcolor[HTML]{E0E0E0}\textbf{\begin{tabular}[c]{@{}l@{}}Student\\ Performance\end{tabular}}}                          & sex                                 & 1e-3                        & 0.05                         & 32                        \\ 
\bottomrule

\end{tabular}}
\vspace{0.7em}
\caption{Hyperparameter settings for standard fairness datasets.}
\label{tab:hyperparameters}
\end{table}

\subsection{Additional Results}
\label{app:additional_result}

\subsubsection{Results with 95\% Confidence Intervals (CIs).}
The main empirical results are averaged from five independent runs using different random seeds, the 95\% CI marked in Table \ref{table:fairness_CI}. In addition to the key findings shown in Table \ref{table:fairness}, we find the results from the two-stage method are often more stable than the one-stage baselines, reflected by tighter CIs.

\begin{table*}[t!]
\centering
\resizebox{\linewidth}{!}{
\begin{tabular}{@{}cc|ccccccc@{}}
\toprule
\multicolumn{1}{c|}{\multirow{2}{*}{}}        & \multicolumn{1}{c|}{\multirow{2}{*}{\textbf{Baseline}}} & \multicolumn{1}{c|}{\multirow{2}{*}{\textbf{DRO}}} & \multicolumn{2}{c|}{\cellcolor[HTML]{e6ddd6}\textbf{LtR}}           & \multicolumn{2}{c|}{\cellcolor[HTML]{ecdcf9}\textbf{FORML}}         & \multicolumn{2}{c}{\cellcolor[HTML]{e7effe}\textbf{Meta-gDRO}} \\ \cmidrule(l){4-9} 
\multicolumn{1}{c|}{}                         & \multicolumn{1}{c|}{}                                   & \multicolumn{1}{c|}{}                              & \multicolumn{1}{C{4.em}}{one-stage}   & \multicolumn{1}{c|}{two-stage (ours)}        & \multicolumn{1}{C{4.em}}{one-stage}   & \multicolumn{1}{c|}{two-stage (ours}        & \multicolumn{1}{C{4.em}}{one-stage}           & two-stage (ours)               \\

\hline \midrule
\multicolumn{9}{l}{\multirow{2}{*}{\textbf{I. Adult Income}~\cite{misc_adult_2}, (sensitive attribute: \textbf{sex})}}\\ \\  \midrule 

\multicolumn{1}{r|}{\textbf{Overall AUC (↑)}} & \multicolumn{1}{c|}{0.778 $\pm$ 0.004}                              & \multicolumn{1}{c|}{0.761 $\pm$ 0.010}                         & \cellcolor[HTML]{e6ddd6}0.830 $\pm$ 0.002 & \multicolumn{1}{c|}{\cellcolor[HTML]{e6ddd6}0.830 $\pm$ 0.004} & 0.801 $\pm$ 0.009 & \multicolumn{1}{c|}{0.810 $\pm$ 0.004} & 0.810 $\pm$ 0.004         & {0.837} $\pm$ 0.004
         \\

\multicolumn{1}{r|}{\textbf{Max-gAUCD (↓)}}   & \multicolumn{1}{c|}{{0.016 $\pm$ 0.012}}                              & \multicolumn{1}{c|}{0.029 $\pm$ 0.016}                         & 0.050 $\pm$ 0.005 & \multicolumn{1}{c|}{0.047 $\pm$ 0.003}  & \cellcolor[HTML]{ecdcf9}0.075 $\pm$ 0.042 & \multicolumn{1}{c|}{\cellcolor[HTML]{ecdcf9}0.031 $\pm$ 0.009} & 0.052 $\pm$ 0.007        & 0.046 $\pm$ 0.004        \\

\multicolumn{1}{r|}{\textbf{Worst-gAUC (↑)}}  & \multicolumn{1}{c|}{0.770 $\pm$ 0.011}                              & \multicolumn{1}{c|}{0.747 $\pm$ 0.018}                         & 0.805 $\pm$ 0.004 & \multicolumn{1}{c|}{0.807 $\pm$ 0.004} & 0.763 $\pm$ 0.022 & \multicolumn{1}{c|}{0.795 $\pm$ 0.007} & \cellcolor[HTML]{e7effe}0.809 $\pm$ 0.006         & \cellcolor[HTML]{e7effe}{0.814 $\pm$ 0.006}   \\


\hline \midrule
\multicolumn{9}{l}{\multirow{2}{*}{\textbf{II. Adult Income}~\cite{misc_adult_2}, (sensitive attribute: \textbf{race};
\scalebox{1}{\colorbox[HTML]{FFEEC2}{* one group in training data contains only one positive (favorable) label.}}
)}} \\ \\  \midrule 

\multicolumn{1}{r|}{\textbf{Overall AUC (↑)}} & \multicolumn{1}{c|}{0.668 $\pm$ 0.004}                              & \multicolumn{1}{c|}{0.652 $\pm$ 0.005}                         & \cellcolor[HTML]{e6ddd6}0.803 $\pm$ 0.005 & \multicolumn{1}{c|}{\cellcolor[HTML]{e6ddd6}{0.805 $\pm$ 0.007}} & 0.710 $\pm$ 0.023 & \multicolumn{1}{c|}{0.775 $\pm$ 0.008} & 0.775 $\pm$ 0.013         & 0.793 $\pm$ 0.004        \\

\multicolumn{1}{r|}{\textbf{Max-gAUCD (↓)}}   & \multicolumn{1}{c|}{0.225 $\pm$ 0.027}                              & \multicolumn{1}{c|}{0.236 $\pm$ 0.011}                         & {0.090 $\pm$ 0.015} & \multicolumn{1}{c|}{{0.090 $\pm$ 0.018}}  & \cellcolor[HTML]{ecdcf9}0.290 $\pm$ 0.006 & \multicolumn{1}{c|}{\cellcolor[HTML]{ecdcf9}0.134 $\pm$ 0.040} & 0.163 $\pm$ 0.044         & 0.158 $\pm$ 0.029        \\

\multicolumn{1}{r|}{\textbf{Worst-gAUC (↑)}}  & \multicolumn{1}{c|}{0.544 $\pm$ 0.019}                              & \multicolumn{1}{c|}{0.538 $\pm$ 0.002}                         & 0.755 $\pm$ 0.004 & \multicolumn{1}{c|}{{0.760 $\pm$ 0.018} } & 0.540 $\pm$ 0.019 & \multicolumn{1}{c|}{0.688 $\pm$ 0.036} & \cellcolor[HTML]{e7effe}0.694 $\pm$ 0.047         & \cellcolor[HTML]{e7effe}0.703 $\pm$ 0.028   \\

\hline \midrule
\multicolumn{9}{l}{\multirow{2}{*}{\textbf{III. Bank Telemarketing}~\cite{moro2014data}, (sensitive attribute: \textbf{age})}}\\ \\  \midrule 

\multicolumn{1}{r|}{\textbf{Overall AUC (↑)}} & \multicolumn{1}{c|}{0.697 $\pm$ 0.004}                              & \multicolumn{1}{c|}{0.686 $\pm$ 0.015}                         & \cellcolor[HTML]{e6ddd6}0.724 $\pm$ 0.009 & \multicolumn{1}{c|}{\cellcolor[HTML]{e6ddd6}0.728 $\pm$ 0.005} & 0.706 $\pm$ 0.101 & \multicolumn{1}{c|}{0.779 $\pm$ 0.027} & 0.698 $\pm$ 0.006         & 0.722 $\pm$ 0.009         \\

\multicolumn{1}{r|}{\textbf{Max-gAUCD (↓)}}   & \multicolumn{1}{c|}{{0.013 $\pm$ 0.010}}                              & \multicolumn{1}{c|}{0.025 $\pm$ 0.010}                         & 0.099 $\pm$ 0.012 & \multicolumn{1}{c|}{0.083 $\pm$ 0.009}  & \cellcolor[HTML]{ecdcf9}0.039 $\pm$ 0.023 & \multicolumn{1}{c|}{\cellcolor[HTML]{ecdcf9}0.029 $\pm$ 0.009} & 0.098 $\pm$ 0.010       & 0.079 $\pm$ 0.008         \\

\multicolumn{1}{r|}{\textbf{Worst-gAUC (↑)}}  & \multicolumn{1}{c|}{0.691 $\pm$ 0.008}                              & \multicolumn{1}{c|}{0.691 $\pm$ 0.008}                         & 0.675 $\pm$ 0.013 & \multicolumn{1}{c|}{0.686 $\pm$ 0.005} & 0.686 $\pm$ 0.090 & \multicolumn{1}{c|}{0.764 $\pm$ 0.025} & \cellcolor[HTML]{e7effe}0.649 $\pm$ 0.009    & \cellcolor[HTML]{e7effe}0.683 $\pm$ 0.010   \\


\hline \midrule
\multicolumn{9}{l}{\multirow{2}{*}{\textbf{IV. Credit Default}~\cite{yeh2009comparisons}, (sensitive attribute: \textbf{sex};
\scalebox{1}{\colorbox[HTML]{FFEEC2}{* val data is more noisy than others, see Table \ref{tab:noisy_test_set}, Appendix \ref{app:additional_result}.}}
)}}\\ \\  \midrule 

\multicolumn{1}{r|}{\textbf{Overall AUC (↑)}} & \multicolumn{1}{c|}{0.634 $\pm$ 0.004}                              & \multicolumn{1}{c|}{0.624 $\pm$ 0.008}                         & \cellcolor[HTML]{FFFFFF}0.630 $\pm$ 0.019 & \multicolumn{1}{c|}{\cellcolor[HTML]{FFFFFF}0.616 $\pm$ 0.035} & 0.554 $\pm$ 0.039 & \multicolumn{1}{c|}{0.611 $\pm$ 0.018} & {0.682 $\pm$ 0.003}        & 0.661 $\pm$ 0.017       \\

\multicolumn{1}{r|}{\textbf{Max-gAUCD (↓)}}   & \multicolumn{1}{c|}{0.024 $\pm$ 0.011}                              & \multicolumn{1}{c|}{0.022 $\pm$ 0.010}                         & 0.037 $\pm$ 0.022 & \multicolumn{1}{c|}{0.025 $\pm$ 0.012}  & \cellcolor[HTML]{FFFFFF} 0.017 $\pm$ 0.012 & \multicolumn{1}{c|}{\cellcolor[HTML]{FFFFFF}0.033 $\pm$ 0.018} & {0.016 $\pm$ 0.007}         & 0.022 $\pm$ 0.016        \\

\multicolumn{1}{r|}{\textbf{Worst-gAUC (↑)}}  & \multicolumn{1}{c|}{0.622 $\pm$ 0.005}                              & \multicolumn{1}{c|}{0.613 $\pm$ 0.012}                         & 0.612 $\pm$ 0.013 & \multicolumn{1}{c|}{0.603 $\pm$ 0.032} & 0.545 $\pm$ 0.033 & \multicolumn{1}{c|}{0.595 $\pm$ 0.010} & \cellcolor[HTML]{FFFFFF}{0.674 $\pm$ 0.004}         & \cellcolor[HTML]{FFFFFF}0.650 $\pm$ 0.025  \\


\hline \midrule
\multicolumn{9}{l}{\multirow{2}{*}{\textbf{V. Communities and Crime}~\cite{misc_communities_and_crime_183}, (sensitive attribute: \textbf{blackgt6pct}; 
\scalebox{1}{\colorbox[HTML]{FFEEC2}{* val data is noisy, meanwhile one group in training data are all positive labels}}
)}}\\ \\  \midrule 

\multicolumn{1}{r|}{\textbf{Overall AUC (↑)}} & \multicolumn{1}{c|}{0.525 $\pm$ 0.008}                              & \multicolumn{1}{c|}{0.568 $\pm$ 0.006}                         & \cellcolor[HTML]{e6ddd6}0.679 $\pm$ 0.022 & \multicolumn{1}{c|}{\cellcolor[HTML]{e6ddd6}0.700 $\pm$ 0.032} & 0.554 $\pm$ 0.010 & \multicolumn{1}{c|}{0.568 $\pm$ 0.018} & 0.686 $\pm$ 0.035        & 0.686 $\pm$ 0.035        \\

\multicolumn{1}{r|}{\textbf{Max-gAUCD (↓)}}   & \multicolumn{1}{c|}{0.050 $\pm$ 0.015}                              & \multicolumn{1}{c|}{0.136 $\pm$ 0.012}                         & 0.071 $\pm$ 0.045 & \multicolumn{1}{c|}{0.129 $\pm$ 0.073}  & \cellcolor[HTML]{FFFFFF}0.107 $\pm$ 0.020& \multicolumn{1}{c|}{\cellcolor[HTML]{FFFFFF}0.136 $\pm$ 0.037} & 0.114 $\pm$ 0.073        & 0.114 $\pm$ 0.073       \\

\multicolumn{1}{r|}{\textbf{Worst-gAUC (↑)}}  & \multicolumn{1}{c|}{0.500 $\pm$ 0.000}                              & \multicolumn{1}{c|}{0.500 $\pm$ 0.000}                         & {0.643 $\pm$ 0.000} & \multicolumn{1}{c|}{0.636 $\pm$ 0.046} & 0.500 $\pm$ 0.000 & \multicolumn{1}{c|}{0.500 $\pm$ 0.000} & \cellcolor[HTML]{e7effe}0.629 $\pm$ 0.042         & \cellcolor[HTML]{e7effe}0.629 $\pm$ 0.042   \\


\hline \midrule
\multicolumn{9}{l}{\multirow{2}{*}{\textbf{VI. Titanic Survival}~\cite{titanic}, (sensitive attribute: \textbf{sex})}}\\ \\ \midrule 

\multicolumn{1}{r|}{\textbf{Overall AUC (↑)}} & \multicolumn{1}{c|}{0.972 $\pm$ 0.000}                              & \multicolumn{1}{c|}{{0.983 $\pm$ 0.012}}                         & \cellcolor[HTML]{e6ddd6}0.967 $\pm$ 0.010 & \multicolumn{1}{c|}{\cellcolor[HTML]{e6ddd6}0.978 $\pm$ 0.010} & 0.950 $\pm$ 0.018 & \multicolumn{1}{c|}{0.972 $\pm$ 0.022} & 0.961 $\pm$ 0.012       & 0.972 $\pm$ 0.016        \\

\multicolumn{1}{r|}{\textbf{Max-gAUCD (↓)}}   & \multicolumn{1}{c|}{0.056 $\pm$ 0.000}                              & \multicolumn{1}{c|}{0.033 $\pm$ 0.024}                         & 0.044 $\pm$ 0.019 & \multicolumn{1}{c|}{0.044 $\pm$ 0.019}  & \cellcolor[HTML]{ecdcf9} 0.033 $\pm$ 0.024 & \multicolumn{1}{c|}{\cellcolor[HTML]{ecdcf9}0.011 $\pm$ 0.019} & 0.033 $\pm$ 0.024        & 0.033 $\pm$ 0.024         \\

\multicolumn{1}{r|}{\textbf{Worst-gAUC (↑)}}  & \multicolumn{1}{c|}{0.944 $\pm$ 0.000}                              & \multicolumn{1}{c|}{{0.967 $\pm$ 0.024}}                         & 0.944 $\pm$ 0.000 & \multicolumn{1}{c|}{0.956 $\pm$ 0.019} & 0.933 $\pm$ 0.019 & \multicolumn{1}{c|}{0.967 $\pm$ 0.024} & \cellcolor[HTML]{e7effe}0.944 $\pm$ 0.000       & \cellcolor[HTML]{e7effe}0.956 $\pm$ 0.019  \\


\hline \midrule
\multicolumn{9}{l}{
\multirow{2}{*}{
\textbf{VII. Student Performance}~\cite{misc_student_performance_320}, (sensitive attribute: \textbf{sex})
}}\\ \\ \midrule

\multicolumn{1}{r|}{\textbf{Overall AUC (↑)}} & \multicolumn{1}{c|}{0.784 $\pm$ 0.011}                              & \multicolumn{1}{c|}{0.816 $\pm$ 0.020}                         & \cellcolor[HTML]{e6ddd6}0.900 $\pm$ 0.007 & \multicolumn{1}{c|}{\cellcolor[HTML]{e6ddd6}0.900 $\pm$ 0.011} & 0.828 $\pm$ 0.026 & \multicolumn{1}{c|}{0.822 $\pm$ 0.032} & 0.909 $\pm$ 0.020       & 0.912 $\pm$ 0.017      \\

\multicolumn{1}{r|}{\textbf{Max-gAUCD (↓)}}   & \multicolumn{1}{c|}{0.119 $\pm$ 0.032}                              & \multicolumn{1}{c|}{0.106 $\pm$ 0.037}                         & {0.013 $\pm$ 0.013} & \multicolumn{1}{c|}{0.037 $\pm$ 0.027}  & \cellcolor[HTML]{ecdcf9}0.056 $\pm$ 0.032 & \multicolumn{1}{c|}{\cellcolor[HTML]{ecdcf9}0.031 $\pm$ 0.025} & 0.031 $\pm$ 0.018       & 0.025 $\pm$ 0.011       \\

\multicolumn{1}{r|}{\textbf{Worst-gAUC (↑)}}  & \multicolumn{1}{c|}{0.725 $\pm$ 0.020}                             & \multicolumn{1}{c|}{0.762 $\pm$ 0.032}                         & 0.894 $\pm$ 0.013 & \multicolumn{1}{c|}{0.881 $\pm$ 0.020} & 0.800 $\pm$ 0.022 & \multicolumn{1}{c|}{0.806 $\pm$ 0.020} & \cellcolor[HTML]{e7effe}0.894 $\pm$ 0.028         & \cellcolor[HTML]{e7effe}0.900 $\pm$ 0.020   \\

\bottomrule
\end{tabular}}
\caption{Performance comparison on standard fairness datasets (averaged from 5 runs), with 95\% CI.
}
\label{table:fairness_CI}
\end{table*}

\subsubsection{Additional Analysis}
\paragraph{Noise analysis.}
To analyze the noise presented in validation sets, as referenced in Table \ref{tab:noisy_test_set}, we employ normalized mutual information \cite{khan2007relative}. This metric offers a measure of the shared information between features and labels, normalized by the sum of their individual entropies. The normalized mutual information \(I_{norm}\) between a feature set \(X\) and labels \(Y\) is computed as:
\[
I_{norm}(X; Y) = \frac{I(X; Y)}{H(X) + H(Y)}
\]
where \(I(X; Y)\) is the mutual information between \(X\) and \(Y\), representing the amount of shared information, and \(H(X)\) and \(H(Y)\) are the entropies of the feature set and the labels, rsp. A higher \(I_{norm}\) indicates less noise, signifying a stronger relationship between the features and labels. 
We referred to \url{https://github.com/mutualinfo/mutual_info} (Apache-2.0 license) for implementation.

Referring to Table \ref{tab:noisy_test_set}, it is evident that the validation set used from Credit Default dataset and the Communities and Crime dataset exhibits a notably lower mutual information score, compared to the other datasets. This suggests a significantly weaker correlation between the features and labels in these particular validation sets and a higher level of noise. Such a disparity in mutual information underscores the challenge in achieving high model performance and could explain the varying degrees of success observed across different fairness interventions within this dataset.

\paragraph{Additional illustration of resolving hypergradient conflict and our dynamics.}
Figure \ref{fig:all_uap} provides additional visualizations of the impact of our bargaining strategy on gradient conflict resolution, utilizing the adult income dataset with sex as the sensitive attribute. The data indicates that FORML and Meta-gDRO exhibit greater initial conflicts compared to LtR, but also receive more pronounced benefits from our bargaining steps. In Figure \ref{fig:alignment}, we observe instances where bargaining was unsuccessful, yet an external force, aligned with specific fairness goals (shown with Meta-gDRO as an example), aids in overcoming local sticking points. These observations not only substantiate the utility of bargaining steps in mitigating conflicts but also empirically demonstrate the dynamic process of opting for updates over stagnation, which aligns with the analysis of our method discussed in $\S$\ref{subsec:nash_meta_dynamics}.

\begin{table}[!ht]
\centering
\resizebox{0.64\linewidth}{!}{
\begin{tabular}{@{}lccc@{}}
\toprule
\multicolumn{1}{c}{{\color[HTML]{000000} }}                                                                                                  & {\color[HTML]{000000} \textbf{Attribute}} & {\color[HTML]{000000} \textbf{\# features}} & {\color[HTML]{000000} \textbf{Normalized Mutual Information (Validation Set)}} \\ \midrule \midrule
\multicolumn{1}{l|}{}                                                                                                                        & sex                                       & 105                                         &  31.007 e-5                                                             \\ \cmidrule(l){2-4} 
\multicolumn{1}{l|}{\multirow{-2}{*}{\textbf{\begin{tabular}[c]{@{}l@{}}Adult \\ Income\end{tabular}}}}                                      & race                                      & 105                                         & 43.337 e-5                                                             \\ \midrule
\rowcolor[HTML]{E0E0E0} 
\multicolumn{1}{l|}{\cellcolor[HTML]{E0E0E0}\textbf{\begin{tabular}[c]{@{}l@{}}Bank\\ Telemarketing\end{tabular}}}                           & age                                       & 51                                          & 679.576 e-5                                                             \\ \midrule
\multicolumn{1}{l|}{\textbf{\begin{tabular}[c]{@{}l@{}}Credit \\ Default\end{tabular}}}                                                      & sex                                       & 24                                          & \cellcolor[HTML]{FFEEC2}\textbf{26.023 e-5}                             \\ \midrule
\rowcolor[HTML]{E0E0E0} 
\multicolumn{1}{l|}{\cellcolor[HTML]{E0E0E0}{\color[HTML]{000000} \textbf{\begin{tabular}[c]{@{}l@{}}Communities\\ and Crime\end{tabular}}}} & {\color[HTML]{000000} blackgt6pact}       & {\color[HTML]{000000} 1929}                 & \cellcolor[HTML]{FFEEC2}{\color[HTML]{000000}  \textbf{9.962 e-5}}                                      \\ \midrule
\multicolumn{1}{l|}{\textbf{\begin{tabular}[c]{@{}l@{}}Titanic\\ Survival\end{tabular}}}                                                     & sex                                       & 1526                                        &  63.911 e-5                                                             \\ \midrule
\rowcolor[HTML]{E0E0E0} 
\multicolumn{1}{l|}{\cellcolor[HTML]{E0E0E0}\textbf{\begin{tabular}[c]{@{}l@{}}Student\\ Performance\end{tabular}}}                          & sex                                       & 58                                          &  253.198 e-5                                                              \\ \bottomrule
\end{tabular}}
\vspace{0.7em}
\caption{Mutual information of each dataset's validation set normalized by the summation of features' and labels' entropy.}
\label{tab:noisy_test_set}
\end{table}

\begin{figure}[!htb]
    \centering
    \begin{subfigure}[b]{0.49\textwidth}
        \centering
        \includegraphics[width=\textwidth]{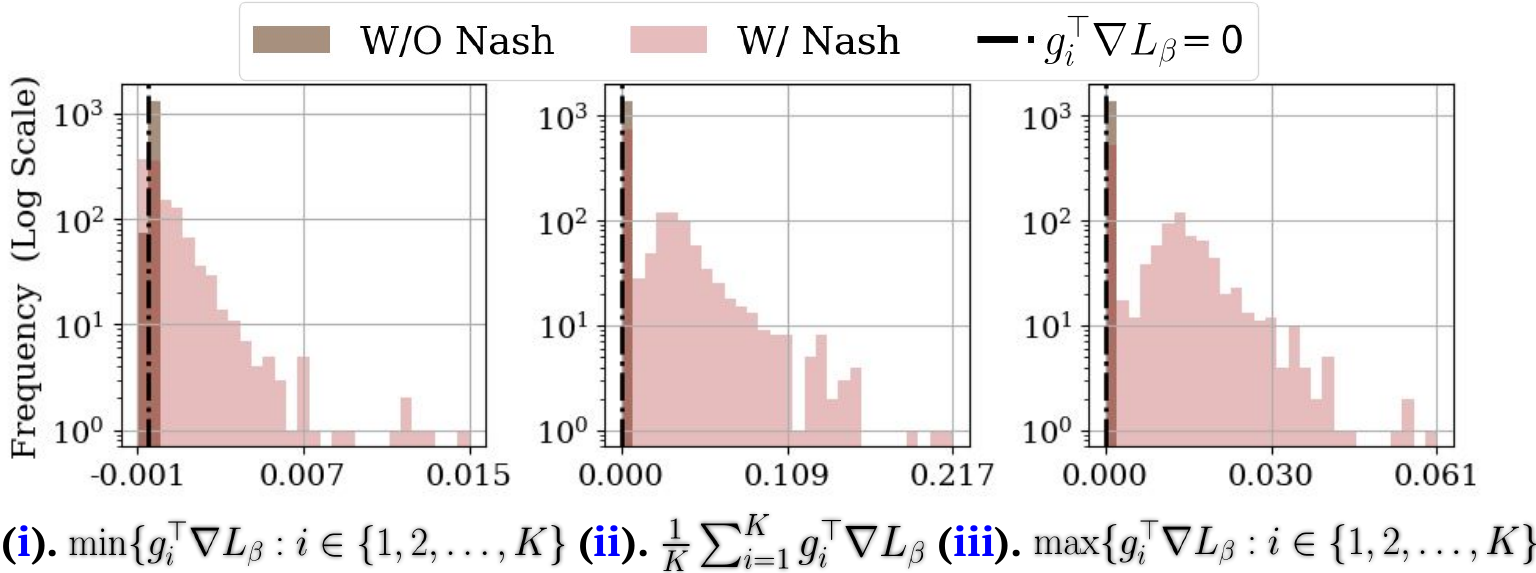}
        \caption{How bargaining helps to resolve conflicts (LtR)}
        \label{fig:uap2}
    \end{subfigure}
    \hfill
    \begin{subfigure}[b]{0.49\textwidth}
        \centering
        \includegraphics[width=\textwidth]{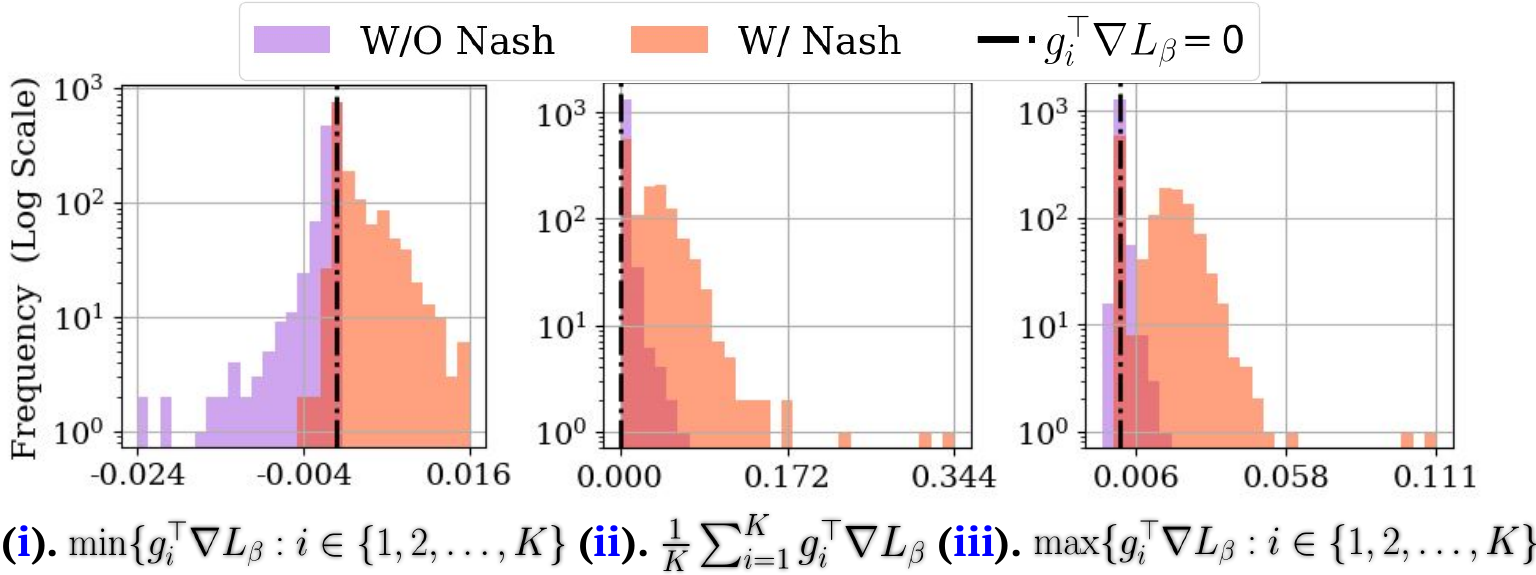}
        \caption{How bargaining helps to resolve conflicts (FORML)}
        \label{fig:uap3}
    \end{subfigure}
    \vspace{\floatsep}
    
    \begin{subfigure}[b]{0.49\textwidth}
        \centering
        \includegraphics[width=\textwidth]{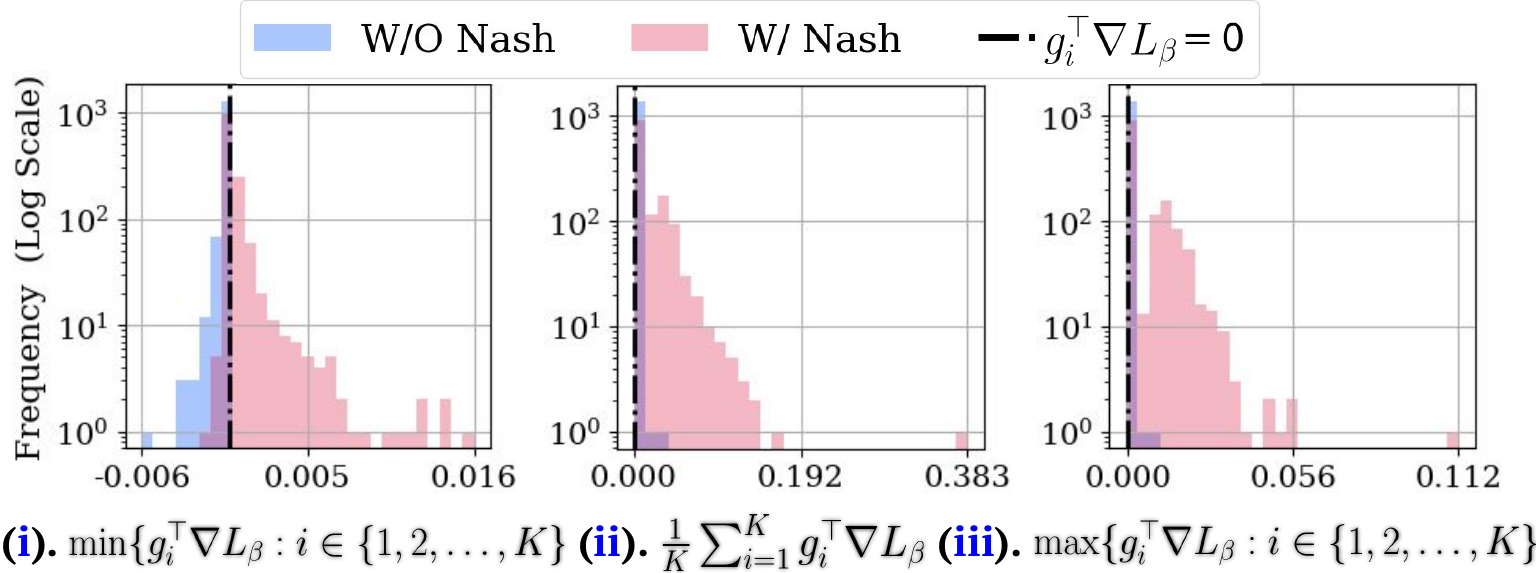}
        \caption{How bargaining helps to resolve conflicts (Meta-gDRO)}
        \label{fig:uap4}
    \end{subfigure}
    \hfill
    \begin{subfigure}[b]{0.49\textwidth}
        \centering
        \includegraphics[width=0.75\textwidth]{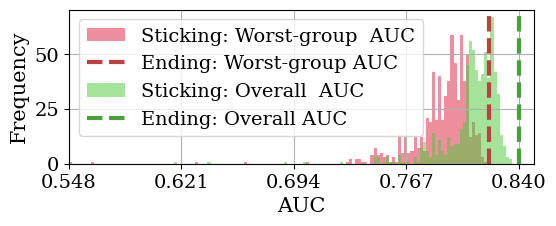}
        \caption{Bargaining failed (local sticking points) and how external force helps (Meta-gDRO with Nash)}
        \label{fig:alignment}
    \end{subfigure}
    \caption{Visualizations of how Nash bargaining improves gradient conflicts across different methods (adult income with sex as the sensitive attribute). With (\textcolor{blue}{i}), (\textcolor{blue}{ii}), (\textcolor{blue}{iii}) in (a), (b), (c) depicting the least value of group-wise hypergradient alignment ($\min\{g_i^{\top}\nabla \fairnessloss : i\in [K]$), the averaged group-wise hypergradient alignment value ($\frac{1}{K}\sum_{i\in [K]} g_i^{\top} \nabla \fairnessloss$), and the maximum value of group-wise hypergradient alignment ($\max\{g_i^{\top}\nabla \fairnessloss : i\in [K]\}$) respectively.
    }
    \label{fig:all_uap}
\end{figure}

\begin{figure}[t!]
    \centering
    \begin{subfigure}[t]{0.49\linewidth}
        \centering
        \includegraphics[width=\linewidth]{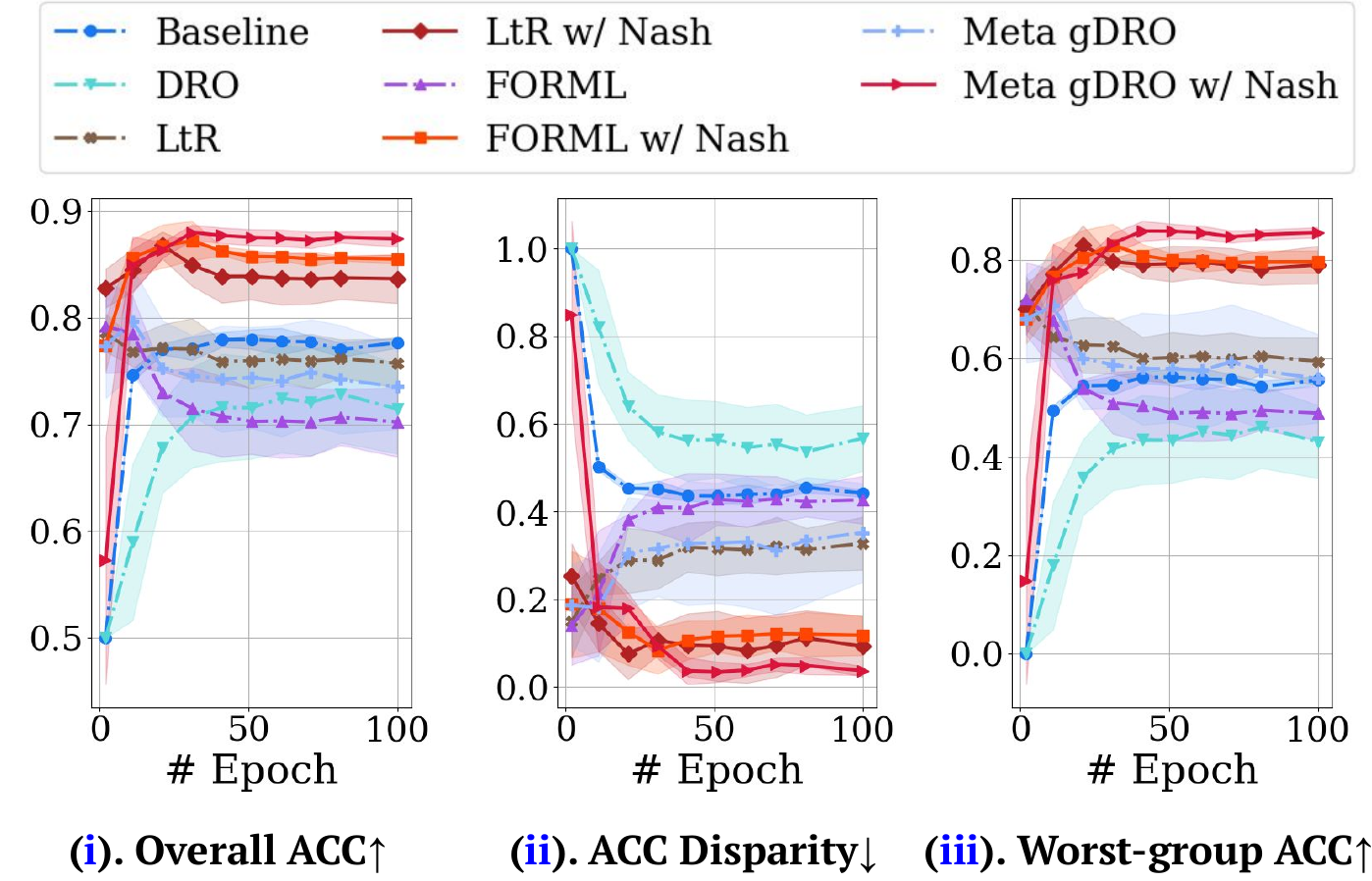}
        \caption{CIFAR-10 imbalanced training results.}
        \label{fig:CIFAR10}
    \end{subfigure}%
    \hfill
    \begin{subfigure}[t]{0.49\linewidth}
        \centering
        \includegraphics[width=\linewidth]{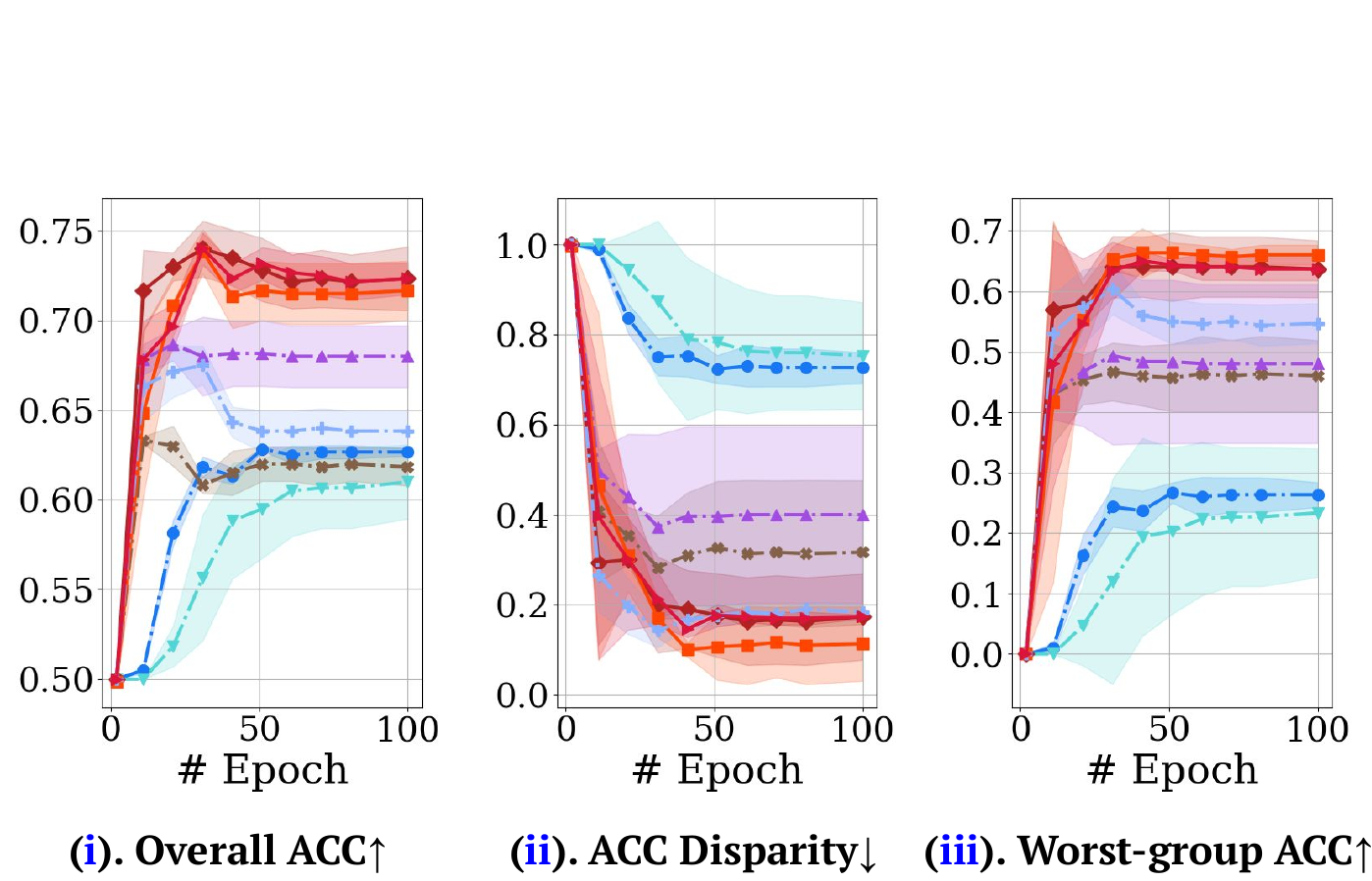}
        \caption{CIFAR-100 imbalanced training results.}
        \label{fig:CIFAR100}
    \end{subfigure}
    \caption{
    Comparative analysis of accuracy and fairness during CIFAR-10 and CIFAR-100 imbalanced training (results averaged from 5 seeds). Subfigures (a) depict the CIFAR-10 results and (b) show CIFAR-100. Each plot tracks the evolution of overall ACC, ACC disparity, and worst-group ACC over training epochs, highlighting performance dynamics under a 99\%-1\% class imbalance for CIFAR-10 and a 95\%-5\% for CIFAR-100. The shaded regions represent the 95\% CI, offering insight into the consistency of each method. The efficacy of Nash bargaining in the two-stage approach is evident, particularly in enhancing the worst-group ACC, vital for fairness in imbalanced settings.
    }
    \label{fig:CIFAR}
\end{figure}

\subsubsection{Imbalanced Image Classification}
\label{subsec:image}

To further validate our framework, we extend our evaluation to imbalanced image classification on subsets of the CIFAR-10 and CIFAR-100 datasets \cite{krizhevsky2009learning}. For CIFAR-10, we construct a training set of 5000 samples with an imbalanced class distribution of 99\% to 1\%. Due to the smaller per-class sample size in CIFAR-100, we create a 500-sample training set with a 95\%-5\% split. Test sets are balanced, containing an equal number of samples from each class (1000 for CIFAR-10, 100 for CIFAR-100), and validation sets are composed of five samples from each group. We utilize ResNet-18 \cite{he2016deep} for modeling and employ SGD with a learning rate of 5e-4, momentum of 0.9, and a weight decay of 5e-4. All methods are trained for 100 epochs with a batch size of 128, and for two-stage algorithms, the initial 30 epochs incorporate bargaining.
Figure \ref{fig:CIFAR10} delineates the trajectory of accuracy (ACC) for different methods during training epochs under a 99\%-1\% class imbalance on CIFAR-10. One-stage meta-learning methods encounter difficulties in converging to a performing and fair model, often underperforming in terms of both ACC and fairness in this context. In contrast, our two-stage method not only achieves significant improvements in ACC and fairness but also enhances algorithmic stability, as indicated by the tighter CIs. The early application of bargaining secures a higher initial ACC, and this advantage is sustained even as the model later transitions to focus on specific fairness objectives.  The consistently improved trajectory, echoing our synthetic example insights (Figure \ref{fig:oursimu}), suggests that reaching the Pareto front and subsequently focusing on fairness does not degrade the model's performance. The above observations underscore the efficacy of the two-stage bargaining approach in managing the trade-offs inherent in fairness-focused tasks.
The challenge of achieving fairness across groups is amplified by the granular nature of CIFAR-100 (more classes and fewer samples per class) (Figure \ref{fig:CIFAR100}). Despite this, the two-stage Nash bargaining approach demonstrates a substantial improvement in aligning model performance with fairness goals.  In contrast, the one-stage baseline methods show limitations in addressing the imbalance, as evidenced by their fluctuating and often lower ACC trajectories. Nash bargaining yields a marked improvement in the worst-group ACC, which is critical in such imbalanced scenarios. This indicates that the bargaining phase helps the model to better navigate the complex decision space of CIFAR-100, providing a more equitable distribution of predictive accuracy among the classes. Notably, the early-stage bargaining not only propels the model towards higher initial accuracy but also prepares it to maintain performance when the fairness objectives become the primary focus in the later stages of training. This strategic approach ensures that the pursuit of fairness does not come at the expense of overall accuracy, a balance that is particularly challenging in the diverse and imbalanced CIFAR-100 environment. The above results from CIFAR-100 reinforce the generalizable effectiveness of Nash bargaining in our proposed two-stage meta-learning. 


\subsection{Broader Impact}
\label{app:broader_impact}
In this study, we present advancements in ML fairness through the development and analysis of new fairness-aware meta-learning methods. Our research strictly utilizes synthetic or publicly accessible open-source datasets, avoiding the use of human subjects. While our method demonstrates notable improvements over existing approaches in various experimental setups, we recognize that it is not infallible. As such, any application of our method in real-world decision-making tasks must be approached with meticulous consideration of the specific context and potential repercussions. It's crucial to understand that achieving ML fairness is an ongoing, collective endeavor within the research community. 

Moreover, we also acknowledge that the fairness-aware meta-learning methods developed in this study are not panaceas for all fairness-related issues in machine learning. There are still several challenges and open research directions in this area, such as addressing fairness in multimodal learning, developing better explainability and interpretability techniques, and ensuring fairness in model adaptation and transfer learning. Furthermore, there is a need for more diverse and representative datasets to evaluate the effectiveness and generalizability of fairness-aware methods. To address these challenges, we call for further research and collaboration among researchers, practitioners, and policymakers to advance the state-of-the-art in ML fairness and to ensure that AI systems are fair, transparent, and accountable for all individuals and groups.


\end{document}